\newif\iffullVer
\fullVerfalse

\RequirePackage{amsmath}
\documentclass[letterpaper,11pt]{article}
\usepackage{amsthm}
\usepackage{amsmath,amssymb, mathtools}
\usepackage{framed,caption,color,url}
\usepackage{standalone}
\usepackage{times}
\usepackage{fullpage}
\usepackage{algorithm2e,algorithmicx,algpseudocode,tikz,wrapfig}
\usepackage{graphicx}               
\usepackage{cite} 
\usepackage{nicefrac} 
\usepackage{paralist} 
\usepackage{boxedminipage} 
\usepackage[pdfstartview=FitH,colorlinks,linkcolor=blue,filecolor=blue,citecolor=blue,urlcolor=blue,pagebackref=true]{hyperref}
\usepackage{upgreek}

\usepackage{centernot}

\renewcommand{\t}{\mathsf{t}}
\renewcommand{\Case}{\mathrm{C}}
\newcommand{\Otilde}{\wt{O}}
\newcommand{\PrTam}{\mathsf{PT}}

\newcommand{\metric}{{{\mathsf{d}}}}

\usepackage[utf8]{inputenc}
\newcommand{\Levy}{Lévy\xspace}

\newcommand{\HD}{\mathsf{HD}}

\newcommand{\OnlTam}{\mathsf{On}\Tam}
\newcommand{\OnTam}{\OnlTam}
\newcommand{\OffTam}{\mathsf{Off}\Tam}

\newcommand{\X}{\cX} 
\newcommand{\Y}{\cY} 
\newcommand{\C}{\cC}

\renewcommand{\H}{\cH}
\newcommand{\error}{\mathsf{Err}}
\newcommand{\Error}{\error}
\newcommand{\Conf}{\mathsf{Conf}}
\newcommand{\conf}{\Conf}

\newcommand{\Risk}{\mathsf{Risk}}

\newcommand{\Tam}{\mathsf{Tam}}

\newcommand{\xDist}{\mathbf{x}}

\newcommand{\tDist}{\mathbf{t}}
\newcommand{\TDist}{\mathbf{T}}

\newcommand{\uDist}{\mathbf{u}}
\newcommand{\vDist}{\mathbf{v}}

\newcommand{\tDistVec}{\ol{\tDist}}

\newcommand{\uDistVec}{\ol{\uDist}}
\newcommand{\vDistVec}{\ol{\vDist}}

\newcommand{\insDist}{\xDist}
\newcommand{\instDist}{\insDist}

\newcommand{\uVec}{\ol{u}}
\newcommand{\vVec}{\ol{v}}

\renewcommand{\th}{^\mathrm{th}}
\newcommand{\twist}[2]{\langle #1 \,\, \| \, {#2}\rangle}

\newcommand{\pfix}[2]{ {#1}_{\leq #2}}

\newcommand{\Exists}{\exists\,}

\newcommand{\fDist}{\mathbf{f}}
\newcommand{\aSF}{\mathsf{a}}
\newcommand{\gSF}{\mathsf{g}}
\newcommand{\hSF}{\mathsf{h}}

\newcommand{\avr}[2]{\ifthenelse{\equal{#2}{}}{\aSF({#1})}{\ifthenelse{\equal{#2}{0}}{\aSF(\emptyset)}{\aSF({#1}_{\leq #2})}}}

\newcommand{\avrMax}[2]{\ifthenelse{\equal{#2}{}}{\aSF^*({#1})}{\ifthenelse{\equal{#2}{0}}{\aSF^*(\emptyset)}{\aSF^*({#1}_{\leq #2})}}}

\newcommand{\avrApp}[2]{\ifthenelse{\equal{#2}{}}{\tilde{\aSF}({#1})}{\ifthenelse{\equal{#2}{0}}{\tilde{\aSF}(\emptyset)}{\tilde{\aSF}({#1}_{\leq #2})}}}

\newcommand{\avrAppMax}[2]{\ifthenelse{\equal{#2}{}}{\tilde{\aSF}^*({#1})}{\ifthenelse{\equal{#2}{0}}{\tilde{\aSF}^*(\emptyset)}{\tilde{\aSF}^*({#1}_{\leq #2})}}}

\newcommand{\ArgMax}[2]{\ifthenelse{\equal{#2}{}}{\hSF({#1})}{\ifthenelse{\equal{#2}{0}}{\hSF(\emptyset)}{\hSF({#1}_{\leq #2})}}}

\newcommand{\AppArgMax}[2]{\ifthenelse{\equal{#2}{}}{\tilde{\hSF}({#1})}{\ifthenelse{\equal{#2}{0}}{\tilde{\hSF}(\emptyset)}{\tilde{\hSF}({#1}_{\leq #2})}}}

\newcommand{\Err}{\mathcal{E}}

\newcommand{\gain}[2]{\ifthenelse{\equal{#2}{}}{\gSF(#1)}{\gSF(#1_{\leq #2})}}
\newcommand{\gainMax}[2]{\ifthenelse{\equal{#2}{}}{\gSF^*(#1)}{\gSF^*(#1_{\leq #2})}}
\newcommand{\gainApp}[2]{\ifthenelse{\equal{#2}{}}{\tilde{\gSF}(#1)}{\tilde{\gSF}(#1_{\leq #2})}}
\newcommand{\gainAppMax}[2]{\ifthenelse{\equal{#2}{}}{\tilde{\gSF}^*(#1)}{\tilde{\gSF}^*(#1_{\leq #2})}}

\newcommand{\pr}[2][]{\Pr_{\ifthenelse{\isempty{#1}}{}{{#1}}}\left[{#2}\right]}

\newcommand{\IdealTam}{\mathsf{IdTam}}
\newcommand{\EffTam}{\mathsf{EffTam}}

\newcommand{\AppTam}{\mathsf{AppTam}}

\newcommand{\problem}{\ensuremath{\mathcal{P}}\xspace}

\newcommand{\e}{\mathrm{e}}

\usepackage{graphicx}

\newcommand{\remove}[1]{}

\newcommand{\ol}{\overline}
\newcommand{\wt}[1]{\widetilde{#1}}

\newcommand{\se}{\subseteq}

\newcommand{\set}[1]{\left\{ #1 \right\}}

\newcommand{\bits}{\{0,1\}}

\newcommand{\To}{\mapsto}

\newcommand{\R}{{\mathbb R}}
\newcommand{\N}{{\mathbb N}}

\newcommand{\trainTam}{\vVec}
\newcommand{\Adv}{\mathsf{A}}

\newcommand{\train}{\uVec}

\newcommand{\cC}{{\mathcal C}}

\newcommand{\cE}{{\mathcal E}}

\newcommand{\cH}{{\mathcal H}}

\newcommand{\cP}{{\mathcal P}}

\newcommand{\cS}{{\mathcal S}}

\newcommand{\cX}{{\mathcal X}}
\newcommand{\cY}{{\mathcal Y}}

\usepackage{upgreek}
\usepackage{bm}

\newcommand{\eps}{\varepsilon}

\newcommand{\poly}{\operatorname{poly}}
\newcommand{\polylog}{\operatorname{polylog}}

\newcommand{\Exp}{\operatorname*{\mathbb{E}}}
\newcommand{\Ex}{\Exp}

\newcommand{\Supp}{\operatorname{Supp}}

\newcommand{\argmax}{\operatorname*{argmax}}

\newtheorem{theorem}{Theorem}[section]
\theoremstyle{plain}
\newtheorem{claim}[theorem]{Claim}

\newtheorem{lemma}[theorem]{Lemma}

\theoremstyle{definition}

\newtheorem{definition}[theorem]{Definition}
\newtheorem{construction}[theorem]{Construction}

\theoremstyle{definition}
\newtheorem{remark}[theorem]{Remark}

\newcommand{\sdotfill}{\textcolor[rgb]{0.8,0.8,0.8}{\dotfill}} 

\makeatletter
\def\th@protocol{%
    \normalfont 
    \setbeamercolor{block title example}{bg=orange,fg=white}
    \setbeamercolor{block body example}{bg=orange!20,fg=black}
    \def\inserttheoremblockenv{exampleblock}
  }
\makeatother
\theoremstyle{protocol}

\newtheorem{proto}[theorem]{Protocol}
\newtheorem{protoc}[theorem]{Protocol}

\newcommand{\namedref}[2]{#1~\ref{#2}}

\newcommand{\torestate}[3]{%
\expandafter \def \csname BBRESTATE #2 \endcsname{#3}
\theoremstyle{plain}
\newtheorem{BBRESTATETHMNUM#2}[theorem]{#1}
\begin{BBRESTATETHMNUM#2}\label{#2}\csname BBRESTATE #2 \endcsname   \end{BBRESTATETHMNUM#2}
\newtheorem*{BBRESTATETHMNONNUM#2}{\namedref{#1}{#2}}
}

\newcommand{\restate}[1]{\begin{BBRESTATETHMNONNUM#1}[Restated] \csname BBRESTATE #1 \endcsname
\end{BBRESTATETHMNONNUM#1}}

\title{Can Adversarially Robust Learning Leverage \\ Computational Hardness? 
}
\author{Saeed Mahloujifar\thanks{\texttt{saeed@virginia.edu}, University of Virginia. Supported by University of Virginia's SEAS Research Innovation Awards.} \and Mohammad Mahmoody\thanks{\texttt{mohammad@virginia.edu},  University of Virginia. Supported by NSF CAREER award CCF-1350939,   and  University of Virginia's SEAS Research Innovation Awards.}}
\iffullVer
\newcommand{\Mnote}[1]{{\color{red} [\bf {Mohammad:}  #1]}}
\newcommand{\Snote}[1]{{\color{blue} [\bf {Saeed:}  #1]}}
\else
\newcommand{\Mnote}[1]{}
\newcommand{\Snote}[1]{}
\fi

\begin{document}
\maketitle

\begin{abstract}
Making learners robust to adversarial perturbation at test time (i.e., evasion attacks) or training time (i.e., poisoning attacks) has emerged as a challenging task. It is known that for some natural settings, \emph{sublinear} perturbations in the  training phase or the testing phase can drastically decrease the quality of the predictions. These negative results, however, are \emph{information theoretic} and only prove the \emph{existence} of such successful adversarial perturbations. A natural question for these settings is whether or not we can make classifiers \emph{computationally} robust to \emph{polynomial-time} attacks.

In this work, we prove strong barriers against achieving such envisioned computational robustness both for evasion and poisoning attacks. In particular,  we show that if the test instances come from a product distribution (e.g., uniform over $\{0,1\}^n$ or $[0,1]^n$, or isotropic $n$-variate Gaussian) and that  there is an initial constant error, then there exists a \emph{polynomial-time} attack that finds adversarial examples of Hamming distance $O(\sqrt n)$. 
For poisoning attacks, we prove that for any   learning algorithm with sample complexity $m$ and any efficiently computable ``predicate'' defining some ``bad'' property $B$ for the produced hypothesis (e.g., failing on a particular test) that happens with an initial constant probability, there  exist \emph{polynomial-time} online poisoning attacks that tamper with $O (\sqrt m)$ many examples, replace them with other correctly labeled examples, and increases the probability of the bad event $B$ to  $\approx 1$. 

Both of our poisoning and  evasion attacks are   \emph{black-box} in how they access  their corresponding components of the system (i.e., the hypothesis, the concept, and the learning algorithm) and make no further assumptions about the classifier or the learning algorithm producing the classifier.

\end{abstract}

\iffullVer
\newpage
\else
\fi
\newpage
\tableofcontents

\section{Introduction}

Making trained classifiers robust to adversarial attacks of various forms has been an active line of research in machine learning recently. Two major forms of attack are the so called ``evasion'' and ``poisoning'' attacks. In an evasion attack, an adversary enters the game during the test phase and tries to perturb the original test instance $x$ into a ``close'' adversarial instance $x'$ that is misclassified by the produced hypothesis (a.k.a. trained model) $h$. In a poisoning attack, the adversary manipulates the training data  into a ``closely related'' poisoned version  with the goal of increasing the risk (or some other closely related property such as failing on a particular example) of the  hypothesis $h$ produced based on the poisoned data.
Starting with  Szegedy et al.~\cite{Szegedy:intriguing}  a race has emerged between evasion attacks that aim to find classified adversarial examples and defences against those attacks~\cite{Evasion:TestTime,biggio2014security,Adversarial::Harnessing,Defenses:Distillation,CarliniWagner,Adversarial::FeatureSqueezing,athalye2018obfuscated}. In another line of work, many papers studied poisoning attacks and defense mechanisms against them~\cite{biggio2012poisoning,awasthi2014powerjournal,xiao2015feature,papernot2016towards,rubinstein2009antidote,shafahi2018poison,koh2017understanding,burkard2017analysis,charikar2017learning,diakonikolas2017statistical,Mahloujifar2018:ALT,diakonikolas2018list,diakonikolas2018sever,prasad2018robust,diakonikolas2018efficient,diakonikolas2017being,diakonikolas2018robustly}. Although, some specific problems (e.g., that of image classification) naturally has got more attention in this line of work, like other works in the theory literature, we approach the robustness problem from a general and fundamental perspective.

\paragraph{Is adversarially robust classification  possible?} Recently, started by Gilmer et el.~\cite{gilmer2018adversarial} and followed by~\cite{fawzi2018adversarial,Adversarial:NIPS,inevitable,mahloujifar2018curse}, it was shown that for many natural metric probability spaces of instances (e.g., uniform distribution over $\bits^n, [0,1]^n$, unit $n$-sphere, or isotropic Gaussian in dimension $n$, all with ``normalized'' Euclidean or Hamming distance) adversarial examples of sublinear perturbations exist for almost all test instances. Indeed, as shown by Mahloujifar, Diochnos, and Mahmoody~\cite{mahloujifar2018curse}, if the instances are drawn from any ``normal \Levy family'' \cite{milman1986asymptotic} of metric probability spaces (that include all the above-mentioned examples), and if there exists an initial non-negligible risk for the generated hypothesis classifier $h$, an adversary can perturb an initial instance $x$ into an adversarial one $x'$ that is only $\approx \sqrt{n}$-far (which is sublinear in $n$) from $x$ and that $x'$ is misclassified. 

In the context of poisoning attacks, some classic results about malicious noise \cite{Valiant::DisjunctionsConjunctions,KearnsLi::Malicious,NastyNoise} could be interpreted as limitations of learning under poisoning. On the positive side, the recent  breakthroughs of Diakonikolas et al.~\cite{diakonikolas2016robust} and Lia et al.~\cite{lai2016agnostic} showed the surprising power of robust inference over poisoned data in \emph{polynomial-time} with error that does \emph{not} depend on the dimension of the instances. These works led to an active line of work (e.g., see \cite{charikar2017learning,diakonikolas2017statistical,diakonikolas2018list,diakonikolas2018sever,prasad2018robust,diakonikolas2018efficient}) exploring the possibility of robust statistics over poisoned data with algorithmic guarantees. In particular \cite{charikar2017learning,diakonikolas2018list} showed how to perform \emph{list-docodable} learning; outputting a set of hypothesis one of which is of high quality, and \cite{diakonikolas2018sever,prasad2018robust} studied robust stochastic convex optimization. 

Studying the power of poisoning \emph{attacks},  the work of Mahloujifar et al.~\cite{mahloujifar2018curse} demonstrated the power  of  poisoning attacks of various forms as long as there is a ``small but non-negligible vulnerability'' in  no attack setting. Namely, assuming that the goal of the adversary is to increase the probability of any ``bad event'' $B$  over the generated hypothesis, it was proved in \cite{mahloujifar2018curse} that the adversary can always increase the probability of $B$ from any non-negligible (or at least sub-exponentially large) probability to $\approx 1$ using sublinear perturbations of the training data. In particular, the adversary can decrease the confidence of the produced hypothesis (to have error at most $\eps$ for a fixed $\eps$), or alternatively it can increase the classification error of a particular instance $x$,  using an adversarial poisoning strategy that achieves these goals by changing $\approx \sqrt m$ of the training examples, where $m$ is the sample complexity of the learner.


\paragraph{Is \emph{computationally} robust classification possible?}  All the above-mentioned sublinear-perturbation attacks of~\cite{fawzi2018adversarial,Adversarial:NIPS,inevitable,mahloujifar2018curse}, in both evasion and poisoning models, were \emph{information theoretic} (i.e., \emph{existential}). Namely, they only show the {existence} of such adversarial instances for evasion attacks or that they show the existence of such adversarial poisoned data with sublinear perturbations for poisoning attacks. In this work, we study the next natural question; can we overcome these information theoretic (existential) lower bounds by relying on the fact that the adversary is computationally bounded? Namely, can we design solutions that resist  \emph{polynomial-time} attacks on the robustness of the learning algorithms? More specifically, the general question studied in our work is as follows.

\begin{quote}
    \emph{Can we make classifiers robust to \emph{computationally bounded} adversarial perturbations  (of sublinear magnitude) that occur during the training  or the test phase?}
\end{quote}

In this work we focus on sublinear perturbations as our main results are \emph{negative} (i.e., demonstrating the power of sublinear tampering).\footnote{For any result proved on the positive side, e.g., it would be stronger to resist even a \emph{linear} amount of perturbations. }

\subsection{Our Results}

In this work, we prove strong barriers against basing the robustness of classifiers, in both evasion  and poisoning settings, on computational intractability. Namely, we show that in many natural settings (i.e., any problem for which  the instances are drawn from a product distribution and that their distances are measured by Hamming distance) adversarial examples could be found in \emph{polynomial time}. This result applies to any learning task over these distributions. In the poisoning attacks' setting, we show that  for any learning task and any distribution over the labeled instances, if the goal of the adversary is to decrease the confidence of the learner or to increase its error on any particular instance $x$, it can always do so in polynomial time by only changing $\approx \sqrt m$ of the labeled instances and replacing them with yet correctly labeled examples. Below we describe both of these results at a high level.

\begin{theorem}[Informal: polynomial-time evasion attacks] \label{thm:EvRob-Inf}
Let $\problem$ be a  classification problem in which the test instances are drawn from a product distribution $\instDist \equiv \uDist_1 \times \dots \times \uDist_n$. 
Suppose $c$ is a concept function (i.e.,  ground truth) and $h$ is a hypothesis that has a constant $\Omega(1)$ error in predicting $c$. Then, there is a \emph{polynomial-time} (black-box) adversary  that perturbs only  $\approx O(\sqrt n)$ of the \emph{blocks} of the instances  and make them misclassified with probability $\approx 1$.
 \end{theorem}
 
 (See Theorem~\ref{thm:EvRob}  for the formal version of the following theorem.)
 
The above theorem covers many natural distributions such as uniform distributions over $\bits^n$ or $[0,1]^n$ or the isotropic Gaussian of dimension $n$, so long as the distance measure is Hamming distance.
Also, as we will see in Theorem~\ref{thm:EvRob}, the initial error necessary for our polynomial-time evasion attack could be as small as $1/\poly(\log n)$ to keep the perturbations $\Otilde(\sqrt n)$, and even initial error $\omega(\log n /\sqrt{ n})$ is enough to keep the perturbations sublinear $o(n)$. Finally, by ``black-box'' we mean that our attacker only needs oracle access to the hypothesis $h$, the ground truth $c$, and distribution $\instDist$.\footnote{As mentioned,  we need to give our adversary oracle access to a sampler for the instance distribution $\xDist$ as well, though this distribution is usually polynomial-time samplable.} This black-box condition is similar to the one defined in previous work of Papernot et al.~\cite{papernot2017practical}, however the notion of black box in some other works (e.g., see~\cite{ilyas2018black}) are more relaxed and give some additional data, such as a vector containing probabilities assigned to each label, to the adversary as well.

We also note that, even though \emph{learning} is usually done with respect to a \emph{family} of distributions (e.g., all distributions), working with a particular distribution in our \emph{negative} results make them indeed \emph{stronger}.

We now describe our main result about polynomial-time poisoning attacks.
See Theorem~\ref{thm:PoRob}  for the formal version of the following theorem.
\begin{theorem}[Informal: polynomial-time poisoning attacks] \label{thm:PoRob-Inf}
Let $\problem$ be a  classification problem with a deterministic learner $L$ that is given $m$ labeled examples of the form $(x,c(x))$ for a concept function $c$ (determining the ground truth). 
\begin{itemize}
\item {\bf Decreasing confidence.} For any risk threshold $\eps \in [0,1]$, let $\rho$ be the probability that $L$ produces a hypothesis of risk at most $\eps$, referred to as the $\eps$-confidence of $L$. If $\rho$ is at most $1-\Omega(1)$, then there is a \emph{polynomial-time} adversary that replaces at most $\approx O(\sqrt m)$ of the training examples with other correctly classified examples and makes the $\eps$-confidence go down to any constant $O(1)\approx 0$.
\item {\bf Increasing chosen-instance\footnote{Poisoning attacks in which the instance is chosen are also called \emph{targeted} \cite{barreno2006can}.} error.} For any fixed test instance $x$, if the average error of the hypotheses generated by $L$ over instance $x$ is at least  $\Omega(1)$, then there is a \emph{polynomial-time} adversary  that replaces at most $\approx O(\sqrt m)$ of the training examples with other correctly classified examples and increases this average error to any constant $ \approx 1$.
\end{itemize}
Moreover, both attacks above are \emph{online} and  \emph{black-box}.
\end{theorem}

\paragraph{Generalization to arbitrary predicates.} More generally, and similarly to the information theoretic attacks of \cite{mahloujifar2018curse}, the two parts of Theorem \ref{thm:PoRob-Inf} follow as  special cases of a more general result, in which the adversary has a particular efficiently checkable \emph{predicate} in mind defined over the  hypothesis (e.g., mislabelling on a particular $x$ or having more than $\eps$ risk). We show that the adversary can significantly increase the probability of this bad event if it originally happens with any (arbitrary small) constant probability. 

Our  negative results do not contradict the recent successful line of work started by \cite{diakonikolas2016robust,lai2016agnostic}, as in our setting, we start with an initial required error in the no-attack scenario and  show that  any such seemingly benign vulnerability (of say probability $1/1000$) can be significantly amplified.

\paragraph{Other features of our poisoning attacks.} Similarly to the previous attacks of Mahloujifar et al.~\cite{pTampTCC17,mahloujifar2018curse,ITCS-sub}, both poisoning attacks of Theorem~\ref{thm:PoRob-Inf} have the following features.
\begin{compactenum}
    \item Our attacks are online; i.e., during the attack, the adversary is only aware of the training examples sampled \emph{so far} when it decides about the next tampering decision. So, these attacks can be launched against online learners in a way that the tampering happens concurrently with the learning process (see \cite{wang2018data} for an in-depth study of attacks against online learners). The information theoretic attacks of \cite{mahloujifar2018curse} were ``off-line'' as the adversary needed the full training sequence before attacking.
    \item Our attacks only use \emph{correct labels} for instances that they inject to the training set (see~\cite{shafahi2018poison} where attacks of this form are studied in practice).
    
    \item Our attacks are black-box~\cite{papernot2017practical}, as they use the learning algorithm $L$ and concept $c$ as oracles.
\end{compactenum}

\paragraph{Further related work.} Computational constraints for robust learning were previously considered by the works of Mahloujifar et al.~\cite{pTampTCC17,Mahloujifar2018:ALT} for poisoning attacks and   Bubeck et al.~\cite{bubeck2018adversarial} for adversarial examples (i.e., evasion attacks). The works of~\cite{pTampTCC17,Mahloujifar2018:ALT} studied so called ``$p$-tampering'' attacks that are online poisoning attacks in which each incoming training example could become tamperable with independent probability $p$ and even in that case the adversary can substitute them with other \emph{correctly labeled} examples. (The independent probabilities of tampering makes $p$-tampering attacks a special form of Valiant's malicious noise model~\cite{Valiant::DisjunctionsConjunctions}.) The works of \cite{pTampTCC17,Mahloujifar2018:ALT} showed that for an initial constant error $\mu$, \emph{polynomial-time} $p$-tampering attacks can decrease the confidence of the learner or alternatively increase a chosen instance's error  by $\Omega(\mu \cdot p)$. Therefore, in order to increase the (chosen instance) error to $50\%$, their attacks needed to tamper with a \emph{linear} number of training examples.  The more recent work of Mahloujifar et al.~\cite{mahloujifar2018curse} improved this attack to use only a sublinear $\sqrt{m}$ number of tamperings at the cost of only achieving information theoretic (exponential time) attacks. In this work, we get the best of both worlds, i.e., polynomial-time poisoning attacks of sublinear tampering budget.

The recent work of Bubeck, Price, and Razenshteyn~\cite{bubeck2018adversarial} studied whether the difficulty of finding robust classifiers is due to information theoretic barriers or that it is due to computational constraints. Indeed,  they showed that (for a broad range of problems with minimal conditions) \emph{if} we assume the existence of robust classifiers  then polynomially many samples would contain enough information for guiding  the learners towards one of those robust classifiers, even though as shown by Schmidt at al.~\cite{schmidt2018adversarially} this could be provably a larger sample complexity than the setting with no attacks. However,~\cite{bubeck2018adversarial} showed that \emph{finding} such classifier might not be efficiently feasible, where efficiency here is enforced by Kearns' statistical query (SQ) model~\cite{kearns1998efficient}. So, even though our work and the work of ~\cite{bubeck2018adversarial} both study computational constraints, the work of ~\cite{bubeck2018adversarial} studied barriers against \emph{efficiently finding} robust classifiers, while we study whether or not robust classifiers exist at all in the presence of \emph{computationally efficient} attackers. In fact, in \cite{diakonikolas2017statistical}  similar computational barriers were proved against achieving robustness in the poisoning attacks in the SQ model  (i.e., information-theoretic optimal accuracy cannot be achieved by an efficient learning SQ algorithm). However, as mentioned, in this work we are focusing on the \emph{efficiency of the attacker} and ask whether or not such computational limitation could be leveraged for robust learning.



\paragraph{Other related definitions of adversarial examples.}
In both of our results (for poisoning and evasion attacks), we use definitions that require \emph{misclassification} of the  test instance as the main goal of the adversary. However, other definitions of adversarial examples are proposed in the literature that coincide with this definition under natural conditions for practical problems of study (such as image classification). 

\emph{Corrupted inputs.} Feige, Mansour, and Schapire \cite{feige2015learning} (and follow-up works of \cite{feige2018robust,attias2018improved}) studied learning and inference in the presence of \emph{corrupted inputs}. In this setting, the adversary can corrupt the test instance $x$ to another instance $x'$ chosen from ``a few'' possible corrupted versions, and then the classifier's job is to predict the label of the \emph{original} uncorrected instance $x$ by getting $x'$ as input.\footnote{The work of \cite{mansour2015robust} also studied robust inference, but with static corruption in which the adversary chooses its corruption before seeing the test instance.}
Inspired by robust optimization \cite{ben2009robust}, the more recent works of Madry et al.~\cite{madry2017towards} and Schmidt et al.~\cite{schmidt2018adversarially} studied adversarial loss (and risk) for corrupted inputs in metric spaces. (One major difference is that now the number of possible corrupted inputs could be huge.) 

\emph{Prediction change.} Some other works (e.g., ~\cite{Szegedy:intriguing,fawzi2018adversarial}) only compare the prediction of the hypothesis over the adversarial example with its own prediction on the honest example (and so their definition is independent of the ground truth $c$). Even though in many natural settings these definitions become very close, in order to prover our formal theorems we use a definition that is based on the ``error region'' of the hypothesis in comparison with the ground truth that is implicit  in~\cite{gilmer2018adversarial,bubeck2018adversarial} and in~\cite{pTampTCC17,Mahloujifar2018:ALT} in the context of poisoning attacks. We refer the reader to the work of Diochnos, Mahloujifar, and Mahmoody~\cite{Adversarial:NIPS} for a taxonomy of these variants and further discussion.

\subsection{Technique: Computational Concentration of Measure in Product Spaces}
In order to prove our Theorems \ref{thm:EvRob-Inf} and \ref{thm:PoRob-Inf}, we make use of ideas developed in  a recent beautiful work of Kalai, Komargodski and Raz.~\cite{RazCoin2018} in the context of attacking coin tossing protocols. In a nutshell, our proofs proceed by first designing new polynomial-time coin-tossing attacks by first carefully changing the model of \cite{RazCoin2018}, and then we show how such coin tossing attacks can be used to obtain evasion and poisoning attacks. Our new coin tossing attacks could be interpreted as  polynomial-time algorithmic proofs for concentration of measure in product distributions under Hamming distance. We can then use such algorithmic proofs instead of the information theoretic concentration results used in~\cite{mahloujifar2018curse}.

To describe our techniques, it is instructive to first recall  the big picture of the polynomial-time poisoning attacks of~\cite{pTampTCC17,mahloujifar2018curse}, even though they needed linear perturbations, before describing how those ideas can be extended to obtain stronger attacks with sublinear perturbations in both evasion and poisoning contexts. Indeed, the core idea there is to model the task of the adversary by a Boolean function $f(\uVec)$ over the training data $\uVec=(u_1,\dots,u_m)$, and roughly speaking define $f(\uVec)=1$ if the training process over $\uVec$ leads to a misclassification by the hypothesis (on a chosen instance) or ``low confidence'' over the produced hypothesis. Then, they showed how to increase the expected value of any such $f$ from an initial constant value $\mu$ to $\mu' \approx \mu +p$ by tampering with $p$ fraction of ``blocks'' of the input sequence $(u_1,\dots,u_m)$.

The more recent work of~\cite{mahloujifar2018curse} improved the bounds achieved by the above poisoning attacks by using an \emph{computationally unbounded} attack who is \emph{more efficient} in its tampering budget and only tampers with a sublinear $\approx \sqrt m$ number of the $m$ training examples and yet increase the average of $f$ from $\mu =\Omega(1)$ to $\mu'\approx 1$. The key idea used in~\cite{mahloujifar2018curse} was to use the concentration of measure in product probability spaces under the Hamming distance~\cite{amir1980unconditional,milman1986asymptotic,mcdiarmid1989method,talagrand1995concentration}. Namely, it is known that for any product space of dimension $m$ (here, modeling the training sequence that is iid sampled) and any initial set $\cS$ of constant probability (here, $\cS=\set{\uVec\mid f(\uVec)=1}$, ``almost all'' of the points in the product space are of distance $\leq O( \sqrt m)$ from $\cS$, and so the measure is concentrated around $\cS$.

\paragraph{Computational concentration of measure.} In a concentrated spaces (e.g., in normal \Levy families) of dimension $n$, for any sufficiently large set $\cS$ (of, say constant measure) the ``typical'' minimum distance of the space points to $\cS$ is sublinear  $o(n)$ ($O(\sqrt n)$ in normal \Levy families). A computational version of this statement shall find such ``close'' points in $\cS$ in polynomial time. The main technical contribution of our work is to prove such computational concentration of measure  for any product distribution under the Hamming distance.
Namely, we prove the following result about biasing Boolean functions defined over product spaces using polynomial time tampering algorithms. (See Theorem~\ref{thm:ProdOnline} for a formal variant.)

\begin{theorem}[Informal: computational  concentration  of  products]\label{thm:ProdOnline-inf} Let $\uDistVec \equiv \uDist_1 \times \dots \uDist_n$ be any product distribution of dimension $n$ and let $f \colon \Supp(\uDistVec) \To \bits$ be any Boolean function with expected value $\Ex[f(\uDistVec)]=\Omega(1)$. Then,  there is a \emph{polynomial-time} tampering adversary who only tampers with $O(\sqrt{n})$ of the blocks of a sample $\uVec\gets \uDistVec$ and increases the average of $f$ over the tampered distribution to $\approx 1$.
\end{theorem}

Once we prove Theorem~\ref{thm:ProdOnline-inf}, we can also use it directly to obtain \emph{evasion} attacks that find adversarial examples, so long as the test instances are drawn from a product distribution and that the distances over the instances are measured by Hamming distance. Indeed, using concentration results (or their stronger forms of isoperimetric inequalities) was the key method used in previous works of ~\cite{gilmer2018adversarial,fawzi2018adversarial,Adversarial:NIPS,inevitable,mahloujifar2018curse} to show the existence of adversarial examples. Thus, our Theorem~\ref{thm:ProdOnline} is a natural tool to be used in this context as well, as it simply shows that similar (yet not exactly equal) bounds to those proved by the concentration of measure can be achieved  algorithmically using  polynomial time adversaries.

\paragraph{Relation to approximate nearest neighbor search.} We note that computational concentration of measure (e.g., as proved in Theorems \ref{thm:ProdOnline-inf} and \ref{thm:ProdOnline} for product spaces under Hamming distance) bears similarities to the problem of ``approximate nearest neighbor'' (ANN) search problem \cite{indyk1998approximate,andoni2006near,andoni2015optimal,andoni2018approximate} in high dimension. Indeed, in the ANN search problem, we are given a set of points $\cP \se \cX$ where $\cX$ is the support set of a metric probability space (of high dimension). We then want to answer  approximate near neighbor queries quickly. Namely, for a given $x \in \cX$, in case there is a point  $y\in \cP$ where $x$ and $y$ are ``close'', the algorithm should return a point $y'$ that is comparably close to $x$. Despite similarities, (algorithmic proofs of) computational concentration of measure are different in two regards: (1) In our case the set $\cP$ could be huge, so it is not even possible to be given as input, but we rather have \emph{implicit} access to $\cP$ (e.g., by oracle access). (2) We are not necessarily looking for point by point approximate solutions; we only need  the \emph{average} distance of the returned points in $\cP$ to be within some (nontrivial) asymptotic bounds.

\subsubsection{Ideas behind the Proof of Theorem~\ref{thm:ProdOnline-inf}} The proof of our Theorem~\ref{thm:ProdOnline} is inspired by the recent work of Kalai et al.~\cite{RazCoin2018} in the context of attacks against coin tossing protocols. Indeed, \cite{RazCoin2018} proved that in any coin tossing protocol in which $n$ parties send a single message each, there is always an  adversary who can corrupt up to $\approx \sqrt{n}$ of the players adaptively and almost fix the output to $0$ or $1$, making progress towards resolving a conjecture of Ben-Or and Linial \cite{ben1989collective}.

At first sight, it might seem that we should be able to directly use the result of \cite{RazCoin2018} for our purposes of proving Theorem \ref{thm:ProdOnline-inf}, as they design adversaries who tamper with $\approx O(\sqrt n)$ blocks of an incoming input and change the average of a Boolean function defined over them (i.e., the coin toss). However, there are two major obstacles against such approach. (1) The attack of \cite{RazCoin2018} is exponential time (as it is recursively defined over the full tree of the values of the input random process), and (2) their attack can not always \emph{increase} the probability of a function $f$ defined over the input, and it can only guarantee that either we will increase this average or decrease it. In fact (2) is \emph{necessary} for the result of \cite{RazCoin2018}, as in their model the adversary has to pick the tampered blocks \emph{before} seeing their contents, and that there are simple functions for which we cannot choose the direction of the bias arbitrarily.  Both of these restrictions are acceptable in the context of \cite{RazCoin2018}, but not for our setting: here we want to \emph{increase} the average of $f$ as it represents the ``error'' in the learning process, and we want polynomial time biasing attacks.


Interestingly, the work of~\cite{RazCoin2018} also presented an alternative simpler proof for a previously known result of Lichtenstein et al.~\cite{lichtenstein1989some} in the context of adaptive corruption in coin tossing attacks. In that special case, the messages sent by parties  only consist of single bits. In the simpler bit-wise setting, it \emph{is} indeed possible to achieve biasing attacks that always increase the average of the output function bit. Thus, there is hope that such attacks could be adapted to our setting, and this is exactly what we do.

To prove Theorem \ref{thm:ProdOnline}, we do proceed as follows.
\begin{enumerate}
    \item We give a new block-wise biasing attack, inspired by the bit-wise attack of \cite{RazCoin2018}, that also always increases the average  of the final output bit. (This is not possible for block-wise model of \cite{RazCoin2018}.)
    \item We show that this attack \emph{can} be approximate in polynomial time. (The block-wise attack of \cite{RazCoin2018} seems inherently exponential time).
    \item We use ideas from the bit-wise attack of \cite{RazCoin2018} to analyze our block-wise attack. To do this, new subtleties arise that can be handled by using stronger forms of Azuma's inequality (see Lemma~\ref{lem:AzumaApp}) as opposed to the ``basic'' version of this inequality used by \cite{RazCoin2018} for their bit-wise attack.
\end{enumerate}


Here, we describe our new attack in a simplified ideal setting in which  we ignore computational efficiency. We will then compare it with the attack of~\cite{RazCoin2018}. The attack has the form that can be adapted to computationally efficient setting by approximating the partial averages needed for the attack. See Constructions~\ref{const:Semi-Poly} and~\ref{const:Poly} for the formal description of the attack in computational settings.

\begin{construction}[Informal:  biasing attack over product distributions] \label{const:Ideal} Let $\uDistVec \equiv \uDist_1 \times \dots \uDist_n$ be a product distribution. Our  (ideal model) tampering attacker $\IdealTam$ is parameterized by $\tau$.
Given a sequence of blocks $(u_1,\dots,u_n)$, $\IdealTam$ tampers with them by reading them one by one (starting from $u_1$) and decides about the tampered values inductively as follows. Suppose $v_1,\dots,v_{i-1}$ are the finalized values for the first $i-1$ blocks (after tampering decisions).
\begin{itemize}
    \item {\bf Tampering case 1.} If there is \emph{some} value  $v_i \in \Supp(\uDist_i)$ such that by picking it, the average of $f$ goes up by at least $\tau$ for the fixed prefix $(v_1,\dots,v_{i-1})$ and for a \emph{random} continuation of the rest of the blocks, then pick $v_i$ as the tampered value for the $i\th$ block.
    \item {\bf Tampering case 2.} Otherwise, if the actual (untampered) content of the $i\th$ block, namely $u_i$, \emph{decreases} the average of $f$ (under a random continuation of the remaining blocks) for the fixed prefix $(v_1,\dots,v_{i-1})$, then ignore the original block $u_i$, and pick some tampered value $v_i \in \Supp(\uDist_i)$ that $v_i$ at least does not decrease the average. (Such $v_i$ always exists by an elementary averaging argument.)
    \item  {\bf Not tampering.} If none of the above cases happen, output the original sample $v_i = u_i$.
\end{itemize} 
\end{construction} 
By picking  parameter $\tau \approx 1/\sqrt n$, we prove that the attack achieves the desired properties of Theorem~\ref{thm:ProdOnline}; Namely, the number of tampered blocks is $\approx O(1/\tau)$, while the bias of $f$ under  attack is $\approx 1$.

\paragraph{}The bit-wise attack of~\cite{RazCoin2018} can be seen as simpler variant of the attack above in which the adversary (also) has access to an oracle that returns the partial averages for random continuation. Namely, in their attack tampering
cases 1 and 2 are combined into one: if the next \emph{bit} can increase (or equivalently, can decrease) the partial averages of the current prefix by $\tau$, then the adversary chooses to corrupt that bit (even without seeing its actual content). The crucial difference between the bit-wise attack of \cite{RazCoin2018} and our block-wise attack of Theorem \ref{thm:ProdOnline-inf} is in tampering case 2. Here we \emph{do} look at the untampered value of the $i\th$ block, and doing so is \emph{necessary} for getting an attack in block-wise setting that biases $f(\cdot)$ towards $+1$.


\paragraph{Extension to general product distributions and  for coin-tossing protocols.}
Our proof of Theorem~\ref{thm:ProdOnline-inf}, and its formalized version Theorem~\ref{thm:ProdOnline}, with almost no changes extend to any \emph{joint} distributions like $\uDistVec \equiv (\uDist_1,\dots,\uDist_n)$ under a proper definition of online tampering in which the next ``untampered'' block is sampled conditioned on the previously tampered blocks chosen by the adversary. This shows that in any $n$ round coin tossing protocol in which each of the $n$ parties  sends exactly one message, there is a \emph{polynomial-time} \emph{strong}  adaptive adversary who corrupts $O(\sqrt n)$ of the parties and biases the output to be $1$ with $99/100$ probability. A strong adaptive adversary, introduced by Goldwasser et al.~\cite{goldwasser2015adaptively}, allows the adversary to see the messages before they are delivered and then  corrupt a party (and change their message) based on their initial messages that were about to be sent. Our result improves a previous result of~\cite{goldwasser2015adaptively} that was proved for \emph{one-round} protocols using \emph{exponential time} attackers. Our attack extends to arbitrary (up to) $n$ round protocols and is also polynomial time. Our results are incomparable to those of~\cite{RazCoin2018}; while they also corrupt up to $O(\sqrt n)$ of the messages, attackers do not see the messages of the parties before corrupting them, but our attackers inherently rely on this information. On the other hand, their bias is \emph{either} towards $0$ or toward $1$ (for the block-wise setting) while our attacks can choose the direction of the biasing.

\section{Preliminaries} \label{sec:prelim}
\paragraph{General notation.} We  use calligraphic letters (e.g., $\cX$) for sets. 
By $u \gets \uDist$ we denote sampling $u$  from the probability distribution $\uDist$. 
For a randomized algorithm $R(\cdot)$, by $y \gets R(x)$ we denote the randomized execution of $R$ on input $x$ outputting $y$. 
By $\uDist \equiv \vDist$ we denote that the random variables $\uDist$ and $\vDist$ have the same distributions. Unless stated otherwise, by using a bar over a variable $\uVec$, we emphasize that it is a vector.  By $\uDistVec \equiv (\uDist_1,\uDist_2,\dots,\uDist_n)$ we refer to a joint distribution over vectors with $n$ components.
For a joint distribution $\uDistVec \equiv (\uDist_1,\dots,\uDist_n)$, we use $\pfix{\uDist}{i}$ to denote the joint distribution of the first $i$ variables $\uDistVec \equiv (\uDist_1,\dots,\uDist_i)$. Also, for a vector $\uVec=(u_1\dots u_n)$ we use $\pfix{u}{i}$ to denote the prefix $(u_1,\dots, u_i)$.
 For a joint distribution $(\uDist,\vDist)$, by $(\uDist \mid v)$ we denote the conditional distribution $(\uDist \mid \vDist = v)$. By $\Supp(\uDist) = \set{u \mid \Pr[\uDist=u]>0}$ we denote the support set of $\uDist$. 
By $T^\uDist(\cdot)$ we denote an algorithm $T(\cdot)$ with oracle access to a sampler for distribution $\uDist$ that upon every query returns a fresh sample from $\uDist$. By $\uDist \times \vDist$ we refer to the product distribution in which $\uDist$ and $\vDist$ are sampled independently.
By $\uDist^n$ we denote the $n$-fold product $\uDist$ with itself returning $n$  iid samples. Multiple instances of a random variable $\instDist$ in the same statement (e.g., $(\instDist,c(\instDist))$ refer to the same sample. By PPT we denote ``probabilistic polynomial time''.


\paragraph{Notation for classification problems.} A classification problem $(\X,\Y,\insDist,\C,\H)$ is specified by the following components. 
The set $\X$ is the set  of possible \emph{instances}, 
$\Y$ is the set of possible \emph{labels}, 
$\insDist$ is a distribution over $\X$,
$\C$ is a class of \emph{concept} functions where $c\in \C$ is always a mapping from $\X$ to $\Y$.
Even though in a learning problem, we usually work with a \emph{family} of distributions (e.g., all distributions over $\X$) here we work with only one distribution $\instDist$. The reason is that our results are \emph{impossibility} results, and proving limits of learning under a known distribution $\instDist$ are indeed stronger results.
We did not state the loss function explicitly, as we work with classification problems.
 For  $x \in \X, c \in \C$,
the \emph{risk} or \emph{error} of a hypothesis $h \in \H$ is equal to $\Risk(h,c) = \Pr_{x \gets\insDist}[h(x) \neq c(x)]$. We are usually interested in learning problems $(\X,\Y,\insDist,\C,\H)$ with a specific metric $\metric$ defined over $\X$ for the purpose of defining risk and robustness under instance perturbations controlled by metric $\metric$. Then, we simply write $(\X,\Y,\insDist,\C,\H,\metric)$ to include $\metric$.  

\remove{
The following definition formalizes poisoning attacks. Our definition is based on the definitions given in \citep{pTampTCC17,Mahloujifar2018:ALT} for the online case and in \citep{mahloujifar2018curse} for the offline case.
\begin{definition}[Tampering poisoning attacks]
 Let $\problem=(\X,\Y,\insDist,\C,\H)$ be a classification problem. For any concept $c \in \C$ and sample complexity $m$, a poisoning adversary $\Adv$ takes as input  a training sequence $\train\gets (\insDist,c(\insDist))^m$ of length $m$ and outputs a modified  training sequence $\trainTam = \Adv(\train)$ of the same length.
We define the following properties for  $\Adv$. 
\begin{itemize}
    \item $\Adv$ is   \emph{online} if it generates the tampered sequence $\trainTam =(v_1,\dots,v_m)$ from the original training sequence $(u_1,\dots,u_m)$ in an online way by picking $v_i$ based on $u_1,\dots,u_i$.
    \item  $\Adv$ is called  \emph{plausible}, if  for all $(x,y) = v_i \in \trainTam$, it holds that $y=c(x)$.
    \item  $\Adv$ has \emph{average tampering budget} $b= b(m) \leq m$, if 
    $$\Ex_{\substack{\train \gets (\insDist,c(\insDist))^m \\ \trainTam \gets \Adv(\train)}}[\HD(\train, \trainTam))] \leq b$$
    where $\HD$ is the Hamming distance for vectors of dimension $m$.
    \item $\Adv$  is \emph{efficient}, if it runs in probabilistic polynomial time over the length of its input; i.e.,   $\poly(m\cdot \ell)$ time where $\ell$ is the total bit-length of any $(x,y) \gets (\insDist,c(\insDist))$.
\end{itemize}
\end{definition}
}

The following definition  based on the definitions given in \cite{pTampTCC17,Mahloujifar2018:ALT,mahloujifar2018curse}.

\begin{definition}[Confidence, chosen-instance error, and their adversarial variants] \label{def:confAvErr}
Let $L$ be a learning algorithm for  a classification problem $\problem=(\X,\Y,\insDist,\C,\H)$,  $m$ be the sample complexity of $L$, and  $c\in \C$ be any concept. We define the (adversarial) confidence function and chosen-instance error as follows.
\begin{itemize}
    \item {\bf  Confidence function.} For any error function $\eps = \eps(m)$, the \emph{adversarial confidence} in the presence of a  adversary $\Adv$ is defined as
$$\conf_\Adv(m,c,\eps)=\Pr_{\substack{\train \gets \left(\instDist,c(\instDist)\right)^m\\ h \gets L(\Adv(\train))}}[\Risk(h,c)< \eps].$$
By $\conf(\cdot)$ we denote the confidence without any attack; namely, $\conf(\cdot) = \conf_I(\cdot)$ for the trivial (identity function) adversary $I$ that does not change the training data.
    \item {\bf Chosen-instance error.} For a fixed test instance $x\in\X$, the \emph{chosen-instance} error (over instance $x$) in presence of a poisoning adversary $\Adv$ is defined as
$$\error_\Adv(m,c,x)=\Pr_{\substack{\train \gets (\instDist,c(\instDist))^m\\ h \gets L(\Adv(\train))}}[h(x)\neq c(x)].$$
By $\error(\cdot)$ we denote the chosen-instance error (over $x$) without any attacks; namely, $\error(\cdot)=\error_I(\cdot)$ for the trivial (identity function) adversary $I$.
\end{itemize}
\end{definition}

\subsection{Basic Definitions for Tampering Algorithms}
Our tampering adversaries follow a close model to that of $p$-budget adversaries defined in \cite{Mahloujifar2018:ALT}. Such adversaries, given a sequence of blocks, select at most $p$ fraction of the locations in the sequence and change their value. The $p$-budget model of \cite{Mahloujifar2018:ALT} works in an online setting in which, the adversary should decide for the $i$th block, only knowing the first $i-1$ blocks. In this work, we define both online and offline attacks that work in a closely related budget model in which we only bound the \emph{expected} number of tampered blocks. We find this notion more natural for the robustness of learners. 

\remove{Other definitions of tampering adversaries used in \cite{pTampTCC17,Mahloujifar2018:ALT,ITCS-sub} worked with definitions that only give the first $i-1$ first blocks to the adversary before it chooses on the tampered values, but in their setting such limitation buys no extra features.}

\begin{definition} [Online and offline tampering]\label{def:tamp}
We define the following two tampering attack models.
\begin{itemize}
    \item {\bf Online attacks.} Let $\uDistVec \equiv \uDist_1\times \dots \times \uDist_n$ be an arbitrary product distribution.\footnote{We restrict the case of online attacks to product distribution as they will have simpler notations and that they cover our main applications, however they can be generalized to arbitrary joint distributions as well with a bit more care.} We call a (potentially randomized and  computationally unbounded) algorithm $\OnlTam$ an  \emph{online tampering} algorithm for $\uDistVec$, if given any $i\in[n]$ and any $\pfix{u}{i} \in \Supp(\uDist_1)\times \dots \times \Supp(\uDist_i)$, it holds that
$$\Pr_{v_i \gets \OnlTam(\pfix{u}{i})}[v_i \in \Supp(\uDist_i)]=1 ~.$$
Namely, $ \OnlTam(\pfix{u}{i})$  outputs (a candidate $i\th$ block) $v_i$ in the support set of $\uDist_i$.\footnote{Looking ahead, this restriction makes our attacks stronger in the case of poisoning attacks by always picking correct lables during the attack.}
    \item {\bf Offline attacks.} For an arbitrary joint distribution $\uDistVec \equiv (\uDist_1\dots,\uDist_n)$ (that might or might not be a product distribution), we call a (potentially randomized and possibly computationally unbounded) algorithm $\OffTam$ an  \emph{offline tampering} algorithm for $\uDistVec$, if given any  $\uVec \in \Supp(\uDistVec)$, it holds that
$$\Pr_{\vVec \gets \OnlTam(\uVec)}[\vVec \in \Supp(\uDistVec)]=1 ~.$$
Namely, given any $\uVec \gets \uDistVec$, $ \OnlTam(\uVec)$ always outputs a vector in $\Supp(\uDistVec)$. 
    \item {\bf Efficiency of attacks.}
    If $\uDistVec$ is a joint distribution coming from a \emph{family} of distributions (perhaps based on the index $n \in N$), 
we call an online or offline tampering algorithm \emph{efficient}, if its running time is $\poly(N)$ where $N$ is the total bit length of  any $\uVec \in \Supp(\uDistVec)$.
    \item {\bf Notation for tampered distributions.} For any joint distribution $\uDistVec$, any $\uVec \gets \uDistVec$, and for any tampering algorithm $\Tam$, by $\twist{\uVec}{\Tam}$ we refer to the distribution obtained by running $\Tam$ over $\uVec$, and by $\twist{\uDistVec}{\Tam}$ we refer to the final distribution by also sampling $\uVec \gets \uDistVec$ at random. More formally,
    \begin{itemize}
        \item  For an offline tampering algorithm $\OffTam$, the distribution $\twist{\uVec}{\OffTam}$ is sampled by simply running $\OffTam$ on the whole $\uVec$ and obtaining the output $(v_1,\dots,v_n) \gets \OffTam(u_1,\dots,u_n)$.
        \item For an online tampering algorithm $\OnlTam$ and input $\uVec=(u_1,\dots,u_n)$ sampled from a \emph{product} distribution $\uDist_1\times \dots \uDist_n$, we obtain the output $(v_1,\dots,v_n)  \gets \twist{\uVec}{\OnTam}  $ \emph{inductively}: for $i \in [i]$, sample $v_i \gets \OnTam(v_1,\dots,v_{i-1},u_i)$.\footnote{By limiting our online attackers to product distributions, we can sample the whole sequence of ``untampered'' values $(u_1,\dots,u_n)$ at the beginning; otherwise, for general random processes in which the distribution of blocks are correlated, we would need to sample $(u_1,\dots,u_n)$ and $(v_1,\dots,v_n)$ \emph{jointly} by sampling $u_i$ \emph{conditioned on} $v_1,\dots,v_{i-1}$.}
    \end{itemize}
    \item {\bf Average budget of tampering attacks.} Suppose $\metric$ is a metric defined over $\Supp(\uDistVec)$. We say an online or offline tampering algorithm $\Tam$ has \emph{average budget} (at most) $b$, if 
    $$\Ex_{\substack{\uVec \gets \uDistVec, \\ \vVec \gets \twist{\uVec}{\Tam}}}[\metric(\uVec,\vVec)] \leq b.$$
    If no metric  $\metric$ is specified, we use Hamming distance over vectors of dimension $n$.
\end{itemize}
\end{definition}

\section{Polynomial-time Attacks from Computational Concentration of Products}

In this section, we will first formally state our main technical tool, Theorem \ref{thm:ProdOnline}, that underlies our polynomial-time evasion and poisoning attacks. Namely, we will prove that product distributions are ``computationally concentrated'' under the Hamming distance, in the sense that any subset with constant probability, is ``computationally close'' to most of the points in the probability space. We will then use this tool to obtain our attacks against learners. We will finally prove our main technical tool.

\begin{theorem}[Computational concentration of product distributions] \label{thm:ProdOnline}
Let $\uDistVec \equiv \uDist_1\times\dots\times\uDist_n$ be any product distribution and $f \colon \Supp(\uDistVec) \To \bits$ be any Boolean function over $\uDistVec$, and let $\mu = \Ex[f(\uDistVec)]$ be the expected value of $f$. Then, for any $\rho$ where $\mu < \rho<1$, there is an \emph{online} tampering algorithm $\OnlTam$ generating the tampering distribution $\vDistVec\equiv \twist{\uDistVec}{\OnlTam}$ with the following properties.
\begin{enumerate}
    \item {\bf Achieved bias.} $\Ex[f(\vDistVec)] \geq \rho$.
        \item {\bf Efficiency.} Having oracle access to $f$ and a sampler for $\uDistVec$,  $\OnlTam=\OnlTam^{f,\uDistVec}$ runs in time $\poly\left(\frac{n \cdot \ell}{\mu \cdot (1-\rho)}\right)$ where $\ell$ is the maximum bit length of  any $u_i \in \Supp(\uDist_i)$ for any $i\in[n]$.
        \item {\bf Average budget.} $\OnlTam=\OnlTam^{f,\uDistVec}$ uses average budget  $({2}/{\mu})\cdot \sqrt{n \cdot \ln ({2}/{(1-\rho)})}$~.

\end{enumerate}
\end{theorem}

In the rest of this section, we will  use Theorem \ref{thm:ProdOnline} to prove limitations of robust learning in the presence of polynomial-time poisoning and evasion attackers. We will  prove Theorem~\ref{thm:ProdOnline} in the next section. 

\paragraph{Range of initial and target error covered by Theorem~\ref{thm:ProdOnline}.}
For any $(1-\rho) = 1/\poly(n), \mu = O(1/\polylog n)$ Theorem~\ref{thm:ProdOnline} uses an average budget of only $\wt{O}(\sqrt n)$. If we start from larger initial error that is still bounded by $ \mu = o(1/\sqrt n)$, the  average budget given by the attacker of Theorem~\ref{thm:ProdOnline} will still be $o(n)$, which is nontrivial as it is still sublinear in the dimension.  However, if we start from $\mu = \Omega(1/\sqrt n)$, we the attacker of Theorem~\ref{thm:ProdOnline} stops to give a nontrivial bound, as the required linear $\Omega(n)$ budget is enough for getting any target error trivially. In contrast, the information theoretic attacks of \cite{mahloujifar2018curse} can handle much smaller initial error all the way to subexponentially small $\mu$. Finding the maximum range of $\mu$ for which computationally bounded attackers can increase the error to $1-1/\poly(n)$ remains open.


\subsection{Polynomial-time Evasion Attacks}

The following definition of robustness against adversarial perturbations of the input is based on the previous definitions used in \cite{gilmer2018adversarial,bubeck2018adversarial,Adversarial:NIPS,mahloujifar2018curse}
in which the adversary aims at \emph{misclassification} of the adversarially perturbed instance by trying to push them into the error region.

We define the following definition for a fixed distribution $\instDist$ (as our negative results are for simplicity stated for such cases) but a direct generalization can be obtained for any \emph{family} of distributions over the instances. Moreover, we only give a definition for the ``black-box'' type of attacks (again because our attacks are black-box) but a more general definition can be given for non-black-box attacks as well.

\begin{definition}[Computational (error-region) evasion robustness] \label{def:EvRob}
Let $\problem=(\X,\Y,\insDist,\C,\H,\metric)$ be a  classification problem.
 Suppose the components of $\problem$ are indexed by $n \in \N$, and let $0<\mu(n) < \rho(n) \leq 1$ for functions $\mu(n)$ and $\rho(n)$ that for simplicity we denote by $\mu$ and $\rho$. 
 We say that the $\mu$-to-$\rho$ \emph{evasion robustness} of $\problem$ is at most $b=b(n)$, if there is a (perhaps computationally unbounded) tampering oracle algorithm $\Adv^{(\cdot)}$ such that for all $h \in \H, c\in \C$ with error region $\cE=\Err(h,c), \Pr[\insDist \in \cE] \geq \mu$, we have the following.
\begin{enumerate}
    \item Having oracle access to $h,c$ and a sampler for $\instDist$, the oracle adversary $\Adv=\Adv^{h,c,\instDist}(x)$ reaches  adversarial risk to at least $\rho$ (for the choice of $c,h$). Namely, $\Pr_{x \gets \instDist}[\Adv^{h,c,\instDist}(x) \in \cE] \geq \rho$.
    
    \item The average budget of the adversary $\Adv=\Adv^{h,c,\instDist}$ (with oracle access to $h,c$ and a sampler for $\instDist$) is at most $b$ for samples  $x \gets \insDist$ and with respect to metric $\metric$.
\end{enumerate}
The $\mu$-to-$\rho$ \emph{computational} evasion robustness of $\problem$ is at most $b=b(n)$, if the same statement holds for an  \emph{efficient} (i.e., PPT) oracle algorithm $\Adv$.
 \end{definition}

\paragraph{Evasion robustness of  \emph{problems} vs. that of  \emph{learners}.}
Computational evasion robustness as defined in Definition  \ref{def:EvRob} directly deals with learning problems regardless of what learning algorithm is used for them. The reason for such a choice is that in this work, we prove \emph{negative} results demonstrating the \emph{limitations} of computational robustness. Therefore,  limiting the robustness of a learning problems \emph{regardless} of their learner is a stronger result. In particular, any negative result (i.e., showing attackers with small tampering budget) about $\mu$-to-$\rho$  (computational) robustness of a learning problem $\problem$, immediately implies that any learning algorithm $L$ for $\problem$ that produces hypothesis with risk $\approx \mu$ can always be attacked (efficiently) to reach adversarial risk $\rho$.

Now we state and prove our main theorem about evasion attacks. Note that the proof of this theorem is identical to the reduction  shown in \cite{mahloujifar2018curse}. The difference is that, instead of using original concentration inequalities, we use our new results about \emph{computational} concentration of product measures under hamming distance and obtain attacks that work in polynomial time.

\begin{theorem}[Limits on computational evasion robustness] \label{thm:EvRob}
Let $\problem=(\X,\Y,\insDist,\C,\H, \metric)$ be a  classification problem in which the instances' distribution $\instDist \equiv \uDist_1 \times \dots \times \uDist_n$ is a product distribution of dimension $n$ and $\metric$ is the Hamming distance over vectors of dimension $n$.  Let $0<\mu=\mu(n) < \rho= \rho(n) \leq 1$ be functions of $n$. Then, the  $\mu$-to-$\rho$ \emph{computational} evasion robustness of $\problem$ is at most 
$$b=({2}/{\mu})\cdot \sqrt{n \cdot \ln ({2}/{(1-\rho)})}.$$
In particular, if $\mu(n)=\omega(\log n/\sqrt{ n })$ and $\rho(n) = 1-1/\poly(n)$, then
$b=o(n)$ is sublinear in $n$, and if $\mu(n)=\Omega(1/\polylog(n))$ and $\rho(n) = 1-1/\poly(n)$, then $b = \Otilde(\sqrt n)$.
 \end{theorem}

\begin{proof}
We first define a Boolean function $f\colon:\X \to [0,1]$ as follows:
$$f(x) =\begin{cases}
1 & c(x) \neq h(x),\\
0 & c(x) = h(x).
\end{cases}
$$
It is clear that $\Ex[f(\insDist)] = \Pr[\insDist\in\cE] \geq \mu$. Therefore, by using Theorem \ref{thm:ProdOnline}, we know there is an tampering algorithm $A_\mu^{f,\insDist}$ that runs in time $\poly(n\cdot \ell/\mu\cdot(1-\rho))$ and increases the average of $f$ to $\rho$ while using average budget at most $(2/\mu)\cdot\sqrt{n\cdot \ln(2/(1-\rho))}$. Note that $A$ needs oracle access to $f(\cdot)$ which is computable by oracle access to $h(\cdot)$ and $c(\cdot)$.
\end{proof}

\begin{remark}[Computationally bounded prediction-change evasion attacks]
As we mentioned in the introduction, 
some  works studying adversarial examples (e.g., ~\cite{Szegedy:intriguing,fawzi2018adversarial}) study robustness by only comparing the prediction of the hypothesis over the adversarial example with its own prediction on the honest example, and  so their definition is independent of the ground truth $c$. (In the terminology of \cite{Adversarial:NIPS}, such attacks are called \emph{prediction-change} attacks.) Here we point out that our biasing attack of Theorem \ref{thm:ProdOnline} can be used to prove limits on the robustness against such evasion attacks as well. In particular, in \cite{mahloujifar2018curse}, it was shown that using concentration of measure, one can obtain existential (information theoretic) prediction-change attacks (even of the ``targeted'' form in which the target label is selected). By combining the arguments of \cite{mahloujifar2018curse} and plugging in our computationally bounded attack of Theorem~\ref{thm:ProdOnline} one can obtain  impossibility results for basing the robustness of hypotheses on computational hardness.
 \end{remark}

\subsection{Polynomial-time Poisoning Attacks}
The following definition formalizes the notion of robustness against computationally bounded poisoning adversaries. Our definition is based on those of \cite{pTampTCC17,Mahloujifar2018:ALT} who studied  online poisoning attacks and that of  \cite{mahloujifar2018curse} who studied offline poisoning attacks.

\begin{definition}[Computational poisoning robustness] \label{def:PoRob}
Let $\problem=(\X,\Y,\insDist,\C,\H)$ be a  classification problem with a learner $L$ of sample complexity $m$. Let $0<\mu=\mu(m) <\rho= \rho(m) \leq 1$ be functions of $m$.

\begin{itemize}
\item {\bf Computational confidence robustness.} For $\eps=\eps(m)$, we say that the $\rho$-to-$\mu$ \emph{$\eps$-confidence robustness} of the learner $L$ is at most $b=b(m)$, if there is a (computationally unbounded) tampering algorithm $\Adv$  such that for all $c \in \C$ for which $\Conf(m,c,\eps) \leq \rho$, the following two conditions hold.
\begin{enumerate}
    \item The average budget of $\Adv=\Adv^{L,c,\instDist}$ (who has oracle access to $L, c$ and a sampler for $\instDist$) tampering with the distribution $(\instDist,c(\instDist))^m$ is at most $b$.
    
    \item The adversarial confidence for $\eps'={99 \cdot \eps}/{100}$ is at most $\conf_\Adv(m,c,\eps') \leq \mu$  when attacked by the oracle adversary $\Adv=\Adv^{L,c,\instDist}$.\footnote{The computationally-unbounded variant of this definition as used in \cite{mahloujifar2018curse} uses $\eps'=\eps$ instead of $\eps'={99 \cdot \eps}/{100}$, but as observed by \cite{Mahloujifar2018:ALT}, due to the computational bounded nature of our attack we need to have a small gap between $\eps'$  and $\eps$.}
\end{enumerate}
The $\rho$-to-$\mu$ \emph{computational} $\eps$-confidence robustness of the learner $L$ is at most $b=b(n)$, if the same statement holds for an  \emph{efficient} (i.e., PPT) oracle algorithm $\Adv$ .

\item {\bf Computational chosen-instance robustness.} For an instance $x \gets \instDist$, we say that the $\mu$-to-$\rho$ \emph{chosen-instance robustness} of the learner $L$ for $x$ is at most $b=b(m)$, if there is a (computationally unbounded) tampering oracle algorithm $\Adv$ (that could depend on $x$) such that for all $c \in \C$ for which $\Err(m,c,x) \geq \mu$, the following two conditions hold.
\begin{enumerate}
    \item The average budget of $\Adv=\Adv^{L, c,\instDist}$ (who has oracle access to $L,c$ and a sampler for $\instDist$)  tampering with the distribution $(\instDist,c(\instDist))^m$ is at most $b$.

    \item Adversary $\Adv=\Adv^{L,c,\instDist}$ increases the chosen-instance error  to  $\Error_\Adv(m,c,x) \geq \rho$. 
\end{enumerate}
The $\mu$-to-$\rho$ \emph{computational} chosen-instance robustness of the learner $L$ for instance $x$ is at most $b=b(n)$, if the same thing holds for an  \emph{efficient} (i.e., PPT) oracle algorithm $\Adv$.

\end{itemize}
\end{definition}

Now we state and prove our main theorem about poisoning attacks. Again, the proof of this theorem is identical to the reduction from  shown in \cite{mahloujifar2018curse}. The difference is that here we use our new results about \emph{computational} concentration of product measures under hamming distance and get attacks that work in polynomial time. Another difference is that our attacks here are online due the online nature of our martingale attacks on product measures.

\begin{theorem}[Limits on computational poisoning robustness] \label{thm:PoRob}
Let $\problem=(\X,\Y,\insDist,\C,\H)$ be a  classification problem with a \emph{deterministic polynomial-time} learner $L$. Let $0<\mu=\mu(m) < \rho= \rho(m) \leq 1$ be functions of $m$, where $m$ is the sample complexity of $L$.
\begin{itemize}
\item {\bf Confidence robustness.} Let $\eps=\eps(m) \geq 1/\poly(m)$ be the risk threshold defining the confidence function. Then, the $\rho$-to-$\mu$ \emph{computational} $\eps$-confidence robustness of the learner $L$  is at most $b=({2}/{(1-\rho)})\cdot \sqrt{m \cdot \ln ({2}/{\mu})}.$

\item {\bf Chosen-instance  robustness.} For any instance $x \gets \instDist$, the $\mu$-to-$\rho$ \emph{computational} chosen-instance robustness of the learner $L$ for $x$ is at most $b=({2}/{\mu})\cdot \sqrt{m \cdot \ln ({2}/{(1-\rho)})}.$
\end{itemize}

In particular, in both cases above if $\mu(m)=\omega( \log m/\sqrt{ m})$ and $\rho(m) = 1-1/\poly(m)$, then $b=o(m)$ is sublinear in $m$, and if $\mu(m)=\Omega(1/\poly(\log m))$ and $\rho(m) = 1-1/\poly(m)$, then $b = \Otilde(\sqrt m)$.

Moreover, the polynomial time attacker $\Adv$ bounding the computational poisoning robustness in both cases above has the following features: \emph{\bf (1)} $\Adv$ is online, and \emph{\bf (2)} $\Adv$ is plausible; namely, it never uses any wrong labels in its poisoned training data.
\end{theorem}

\begin{proof}
We first prove the case of chosen-instance robustness. We define a Boolean function $f\colon\X^m \to [0,1]$ as follows:
$$f_1(x_1,\dots,x_m) =\begin{cases}
1 & h=L((x_1,c(x_1)),\dots,(x_n,c(x_m)) \land h(x)\neq c(x),\\
0 & h=L((x_1,c(x_1)),\dots,(x_n,c(x_m)) \land h(x)= c(x).
\end{cases}
$$
It is clear that $\Ex[f_1(\insDist^m)] = \Err(m,c,x) \geq \mu$. Therefore, by using Theorem \ref{thm:ProdOnline}, we know there is a PPT tampering Algorithm $A_2^{f_1(\cdot),\mu}$ that runs in time $\poly(m\cdot \ell/(\mu\cdot(1-\rho)))$, and increase the average of $f_1$ to $\rho$ while using average budget at most $(2/\mu)\cdot\sqrt{m\cdot \ln(2/(1-\rho))}$. Note that $A_1$ needs oracle access to $f_1(\cdot)$ which is computable by oracle access to the learning algorithm $L(\cdot)$ and concept $c(\cdot)$.
Now we prove the case of confidence robustness. Again we define a Boolean function $f_2\colon\X^m \to [0,1]$ as follows: 
$$f_2(x_1,\dots,x_m) =\begin{cases}
1 & h=L((x_1,c(x_1)),\dots,(x_n,c(x_m)) \land \Pr[h(\insDist)\neq c(\insDist)] \geq \eps,\\
0 & h=L((x_1,c(x_1)),\dots,(x_n,c(x_m)) \land \Pr[h(\insDist)\neq c(\insDist)] < \eps.
\end{cases}
$$
We have $\Ex[f_2(\insDist^m)] = 1-\Conf(m,c,\eps) \geq 1-\rho$. Therefore, by using Theorem \ref{thm:ProdOnline}, we know there is a PPT tampering Algorithm $A_2^{f_2(\cdot),\mu}$ that runs in time $\poly(m\cdot \ell/(1-\rho)\cdot\mu)$, and increase the average of $f_2$ to $1-\mu$ while using average budget at most $(2/(1-\rho))\cdot\sqrt{m\cdot \ln(2/\mu)}$. Note that $A_2$ needs oracle access to $f_2(\cdot)$, which requires the adversary to know the exact error of a hypothesis. Computing the exact error is not possible in polynomial time but using an emprical estimator, the adversary can find an approximation of the error which is sufficient for the attack (See Corollary 3 of \cite{Mahloujifar2018:ALT}). 
\end{proof}



\section{Products  are Computationally Concentrated under Hamming Distance}
In this section, we formally prove Theorem \ref{thm:ProdOnline}. 
For simplicity of presentation, we  will prove the following theorem for product distributions $\uDistVec \equiv \uDist^n$ over the same $\uDist$, but the same proof directly holds for more general case of $\uDistVec \equiv \uDist_1 \times \dots \uDist_n$.

We will first present an attack in an idealized model in which the adversary has access to some 
promised oracles that approximate certain properties of the function $f$ in a carefully defined way. In this first step, we indeed show that our attack (and its proof) are robust to such approximations. We then show that these promised oracles can be obtained with high probability, and by doing so we obtain the final polynomial time biasing attack proving the concentration of product distributions under Hamming distance.

\subsection{Biasing Attack Using Promised Approximate Oracles}

We first state a usefull lemma that is similar to Azuma inequality but works with approximate martingales.
\begin{lemma}[Azuma's inequality for approximate  conditions] \label{lem:AzumaApp} Let $\tDistVec \equiv (\tDist_1,\dots,\tDist_n)$ be a sequence of $n$ jointly distributed  random variables such that for all $i \in [n]$, $\Pr[|\tDist_i|> \tau] \leq \gamma$ and for all $\pfix{t}{i-1}\gets \pfix{\tDist}{i-1}$, we have 
$\Ex[\tDist_i\mid \pfix{t}{i-1}] \geq -\gamma.$
Then, we have
$$\Pr\left[\sum_{i=1}^n \tDist_i \leq -s \right] \leq \e^{\frac{-(s-n\cdot\gamma)^2}{2n\cdot(\tau+\gamma)^2}} + n\cdot\gamma$$
\end{lemma}

\begin{proof} 
If we let $\gamma=0$, Lemma \ref{lem:AzumaApp} becomes the standard version of Azuma inequality. Here we sketch why Lemma \ref{lem:AzumaApp} can also be reduced to the case that $\gamma=0$ (i.e., Azuma inequality). We build a sequence $\tDist'_i$ from $\tDist_i$ as follows: Sample $t_i \gets \tDist_i \mid \pfix{t'}{i-1}$, if $|t_i+\gamma| \leq \tau + \gamma$, output $t'_i=t_i+\gamma$. Otherwise output 0. We clearly have $\Ex[\tDist'_i \mid \pfix{t'}{i-1}]\geq 0$ and $\Pr[|\tDist'_i| \geq \tau + \gamma]=0$. Now we can use Lemma \ref{lem:AzumaApp} for the basic case of $\gamma=0$ for the sequence $\tDist'_i$ and use it to get a looser bound for sequence  $\tDist_i$, using the fact that $\exists i\in [n], |t_i| \geq \tau$ happens with probability at most $n\cdot \gamma$.
\end{proof}

Now we define some oracle functions that our tampering attack is based on.

\begin{definition}[Notation for oracles]
\label{defs:biasingApproximate} Suppose $f \colon \Supp(\uDistVec) \To \R$ is defined over 
a product distribution $\uDistVec \equiv \uDist_1\times \dots \times \uDist_n$ of dimension $n$. Then, given a specific parameter $\gamma \in [0,1]$ we define the following \emph{promise} oracles for any $i \in [n]$ and any $\pfix{u}{i}\in \Supp(\pfix{\uDist}{i})$. Namely, our promise oracles could be one out of any oracles that satisfy the following guarantees.
\begin{itemize}
\item Oracle $\avr{u}{i}$ returns the average gain conditioned on the given prefix:
$$ \avr{u}{i}= \Ex_{(u_{i+1},\dots,u_n)\gets \uDist_{i+1}\times \dots \times \uDist_n}[f(u_1,\dots,u_n) ].$$
\item Oracle $\gain{u}{i} $ returns the gain on the average in the last block and is defined as
    $$ \gain{u}{{ i}}= \avr{u}{i}-\avr{u}{i-1} .$$
    \item Oracle $ \gainApp{\cdot}{}$ approximates the gain of average in the last block, $|\gainApp{u}{i}-\gain{u}{i}| \leq \gamma.$
    \item Oracle $\gainAppMax{\cdot}{}$ returns the approximate maximum gain with two promised properties:

    $$\text{{[Property A:]~ }} \Pr_{u_i \gets \uDist}\big[\gain{u}{i} > \gainAppMax{u}{i-1} + 2\gamma\big] < \gamma,$$
    $$\text{{[Property B:]~~~~~~~~~~~~~~~~~~~ }}\gainAppMax{u}{i-1}\geq -2\gamma. \text{~~~~~~~~~~~~~~~}$$
    
    \item Oracle $\AppArgMax{\cdot}{}$ returns a sample producing the approximate maximum gain $\gainAppMax{\cdot}{}$. Namely, 
    \remove{if $\AppArgMax{u}{i-1}=w_i$ and if  $\alpha=\gain{\pfix{u}{i-1},w_i}{}$ is the real gain (of block $w_i$) for prefix $(\pfix{u}{i-1},z)$, then we have
        $$\Pr_{u_i \gets \uDist}\big[\gain{u}{i} > \alpha + \gamma\big] < \gamma.$$}
    $$\gainAppMax{u}{i-1} = \gainApp{\pfix{u}{i-1},\AppArgMax{u}{i-1}}{}.$$    

\end{itemize}
\end{definition}


Following is the construction of our tampering attack based on the oracles defined above.

\begin{construction}[Attack using promised approximate oracles] \label{const:Semi-Poly} 
For a product distribution $\uDistVec \equiv \uDist^n$ and $\tau \in [0,1]$, our  (online) efficient tampering attacker $\AppTam_{(\tau,\gamma)}$ is parameterized by $\tau,\gamma\in[0,1]$, but for simplicity it will be denoted as $\AppTam$. The parameter $\gamma$ determines the approximation promised by the oracles used by $\AppTam$. Given $(v_1,\dots,v_{i-1},u_i)\in \Supp(\uDist)^i$ as input, $\AppTam$ will output some $v_i \in \Supp(\uDist)$. Let $w_i =\AppArgMax{v}{i-1}$ as defined in Definition~\ref{defs:biasingApproximate}.  $\AppTam$ chooses its output $v_i$ as follows.
\begin{itemize}
    \item {\bf Tampering.} If $\gainAppMax{\pfix{v}{i-1}}{} \geq \tau$ or if $\gainApp{\pfix{v}{i-1},u_i}{} \leq -\tau$, then output $v_i = w_i$. 
    \item  {\bf Not tampering.} Otherwise, output the original sample $v_i = u_i$.
\end{itemize}
\end{construction}
\remove{\Mnote{Is the one above OK to use in your proof or you prefer the below version?}
\begin{construction} 
For a product distribution $\uDistVec \equiv \uDist^n$ and $\tau \in [0,1]$, our ideal (online) $(\tau,k)$-tampering attacker $\IdealTam$ is defined as follows. Given $(u_1,\dots,u_i)$ as input, we want to find the output some $v_i \in \Supp(\uVec)$. The attack samples $k$ blocks, $y^1,\dots y^k$ from $\uDist$. If for any of them $\avr{\pfix{u}{i},y^j}{}\geq \avr{u}{i} + \tau$, output $y^j$. Otherwise, sample a new $y\gets \uDist$, if $\avr{\pfix{u}{i},y}{} \geq \avr{u}{i} - \tau$,  output $y$ otherwise output $\argmax_{y^j} (\avr{\pfix{u}{i},y^j}{})$
\end{construction}}

Before proving the bounds, we first define some events based on the conditions/cases that happen during the tampering attack of Construction \ref{const:Semi-Poly}.

\begin{definition} \label{def:polytimeCases}
We define the following three Boolean functions over $\cup_{i=1}^n \Supp(\uDist)^i$ based on the actions taken by the tampering algorithm of Construction \ref{const:Semi-Poly}. Namely, for any $(\pfix{v}{i-1},u_i) \in \Supp(\uDist)^i$, we define
$$\Case_1(\pfix{v}{i-1})=\begin{cases}
1 & \text{if $\gainAppMax{v}{i-1}\geq \tau$,}\\
0 & \text{otherwise;}\\
\end{cases}$$
$$\Case_2 (\pfix{v}{i-1},u_i) =\begin{cases}
1 & \text{if  $\Case_1(\pfix{v}{i-1})=0 \land \gainApp{\pfix{v}{i-1},u_i}{} \leq -\tau$,}\\
0 & \text{otherwise;}\\
\end{cases}$$
$$\Case_3 (\pfix{v}{i-1},u_i) =\begin{cases}
1 & \text{if $\Case_1(\pfix{v}{i-1})=0$ and  $\Case_2(\pfix{v}{i-1},u_i)=0$,}\\
0 & \text{otherwise.}\\
\end{cases}$$
Thus, if $\Case_1$ or $\Case_2$ happens, it means that the adversary has chosen to tamper with  block $i$, and if $\Case_3$ happens it means that the adversary has not chosen to tamper with block $i$. Also, since the above functions are Boolean, we might treat them as events as well. 
Moreover, for convenience we define the set $\Case_1 = \set{\pfix{v}{i} \mid i \in[n-1], \pfix{v}{i} \in \Supp(\uDist)^{i} \land \Case_1(\pfix{v}{i})}$.
\end{definition}

The following Claim bounds the average of the function when the attack of Construction \ref{const:Semi-Poly} is performed on the distribution.

\begin{claim}\label{clm:AppBias} If $\vDistVec \equiv \twist{\uDist^n}{\AppTam}$ is the tampering distribution of the efficient attacker $\AppTam$ of Construction~\ref{const:Semi-Poly}, then it holds that
$$\Ex[f(\vDistVec)] \geq 1 -  \e^{\frac{-(\mu-2n\cdot \gamma)^2}{2n\cdot(\tau+4\gamma)^2}} - n\cdot \gamma.$$
\end{claim}
\begin{proof}
Define a function $t$ as follows, 
$$t(\pfix{v}{i-1},u_i)=\begin{cases}
0 & \text{if $\Case_1(\pfix{v}{i-1}) $ or $\Case_2(\pfix{v}{i-1},u_i)$,}\\
\gainApp{\pfix{v}{i-1},u_i}{}& \text{if $\Case_3(\pfix{v}{i-1},u_i)$.}\\
\end{cases}$$
Now consider a sequence of random variables $\tDistVec=(\tDist_1,\dots,\tDist_n)$ sampled as follows. We first sample $\uVec \gets \uDistVec$, then $\vVec\gets \twist{\uVec}{\AppTam}$, and then  $t_i=\t(\pfix{v}{i-1},u_i)$ for $i\in[n]$. 
Now, for any $\pfix{t}{i-1}\gets \pfix{\tDist}{i-1}$, we claim that $\Ex[\tDist_i\mid \pfix{t}{i-1}] \geq 0$. The reason is as follows.
\remove{Now consider a sequence of random variables $(\tDist_1,\dots,\tDist_n)$ such that $\tDist_i\equiv t(\pfix{\vDist}{i})$, for the same $\vDistVec$. For any $\pfix{t}{i}\in \Supp{\pfix{\tDist}{i}}$ We have}
\begin{align*}
\Ex[\tDist_i\mid\pfix{t}{i-1}] &= \Ex_{\pfix{v}{i-1} \gets \pfix{\vDist}{i-1} \mid \pfix{t}{i-1}}\left[\Ex_{u_i \gets \uDist}\left[\gainApp{\pfix{v}{i-1},u_i}{}\cdot \Case_3(\pfix{v}{i-1},u_i)\right]\right]\\
&\geq\Ex_{\pfix{v}{i-1} \gets \pfix{\vDist}{i-1} \mid \pfix{t}{i-1}}\left[\Ex_{u_i \gets \uDist}\left[\gainApp{\pfix{v}{i-1},u_i}{}\cdot \left(\Case_3\left(\pfix{v}{i-1},u_i\right)\lor \Case_2\left(\pfix{v}{i-1},u_i\right)\right)\right]\right]\\
&= \Ex_{\pfix{v}{i-1} \gets \pfix{\vDist}{i-1} \mid \pfix{t}{i-1}}\left[\Ex_{u_i \gets \uDist}\left[\gainApp{\pfix{v}{i-1},u_i}{}\cdot \left(1-\Case_1\left(\pfix{v}{i-1}\right)\right)\right]\right]\\
&= \Ex_{\pfix{v}{i-1} \gets \pfix{\vDist}{i-1} \mid \pfix{t}{i-1}} \left[\left(1-\Case_1\left(\pfix{v}{i-1}\right)\right)\cdot\Ex_{u_i \gets \uDist}\left[\gainApp{\pfix{v}{i-1},u_i}{}\right]\right]\\
&\geq \Ex_{\pfix{v}{i-1} \gets \pfix{\vDist}{i-1} \mid \pfix{t}{i-1}} \left[\left(1-\Case_1\left(\pfix{v}{i-1}\right)\right)\cdot\Ex_{u_i \gets \uDist}\left[\gain{\pfix{v}{i-1},u_i}{} - \gamma\right]\right]\\
&\geq -\gamma.
\end{align*}
Moreover, for any $\pfix{t}{i-1}\in[\Supp(\pfix{\tDist}{i-1})]$ we have 
\begin{align*}
\Pr_{t_i \gets \tDist_i \mid \pfix{t}{i-1}}[|t_i| \geq \tau + 3\gamma] &= \Pr_{\pfix{v}{i} \gets\pfix{\vDist}{i} \mid \pfix{t}{i-1}}[|\gainApp{v}{i}| \geq \tau + 3\gamma \land \Case_3(\pfix{v}{i})]\\
&= \Pr_{\pfix{v}{i} \gets\pfix{\vDist}{i} \mid \pfix{t}{i-1}}[\gainApp{v}{i} \geq \tau + 3\gamma \land \Case_3(\pfix{v}{i})]\\
&\leq \Pr_{\pfix{v}{i} \gets\pfix{\vDist}{i} \mid \pfix{t}{i-1}}[\gainApp{v}{i} \geq \tau + 3\gamma \land \overline{\Case_1(\pfix{v}{i})}]\\
&= \Pr_{\pfix{v}{i} \gets\pfix{\vDist}{i} \mid \pfix{t}{i-1}}[\gainApp{v}{i} \geq \tau + 3\gamma \land \gainAppMax{\vVec}{i-1} \leq \tau]\\
&\leq \Pr_{\pfix{v}{i} \gets\pfix{\vDist}{i} \mid \pfix{t}{i-1}}[\gain{v}{i} \geq \tau +2\gamma \land \gainAppMax{\vVec}{i-1} \leq \tau]\\
&\leq \gamma.
\end{align*}
Therefore, the sequence $\tDistVec=(\tDist_1,\dots,\tDist_n)$, computed over the same $\vVec \gets \vDistVec$, satisfies the properties required in Lemma \ref{lem:AzumaApp}. (Let  $\tau$ of that Lemma \ref{lem:AzumaApp} to be $\tau+3\gamma$ here, and and letting $s$ of that lemma to be  $-\mu + 2n\cdot \gamma$.) This way, we get
$$\Pr\Big[\sum_{i=1}^{n} \tDist_i \leq -\mu  +2n\cdot\gamma\Big] \leq  \e^{\frac{-(\mu-3n\gamma)^2}{2n\cdot(\tau+4\gamma)^2}} + n\cdot \gamma~.$$
On the other hand, for every $\vVec\in \Supp(\vDistVec)$ we have 
\begin{align*}
f(\vVec) &= \mu + \sum_{i=1}^n \gain{v}{i}\\
&\geq \mu + \sum_{i=1}^n (\gainApp{v}{i} -\gamma)\\
 &= \mu - n\cdot \gamma + \sum_{i=1}^n \left(\Case_1\left(\pfix{v}{i-1}\right) +\Case_3\left(\pfix{v}{i}\right)\right)\cdot\gainApp{v}{i} \\
&\geq \mu -n\cdot\gamma + \sum_{i=1}^{n} t_i - \sum_{i=1}^{n} \Case_1(\pfix{v}{i-1})\cdot \gamma \\
&\geq \mu - 2n\cdot \gamma + \sum_{i=1}^{n} t_i .
\end{align*}
Therefore, we have
$$\Pr[f(\vDistVec)=0] \leq  \Pr\Big[\sum_{i=1}^{n} \tDist_i \leq -\mu + 2n\cdot\gamma\Big] \leq  \e^{\frac{-(\mu-3n\cdot \gamma)^2}{2n\cdot(\tau+4\gamma)^2}} + n\cdot \gamma.$$
\end{proof}

Now we state and prove another Claim which bounds the expected number of tamperings performed by the attack of Construction \ref{const:Semi-Poly}

\begin{claim} \label{clm:AppTam2}
For a tampering sequence $\vVec \gets \twist{\uVec}{\AppTam}$, let  $T_i = \Case_1(\pfix{v}{i-1}) \lor \Case_2(\pfix{v}{i-1},u_i)$ be the event (or equivalently the  Boolean function) denoting that a tampering choice is made by the adversary of Construction \ref{const:Semi-Poly} over the $i$'th block. If $T=\sum_{i=1}^n T_i$ denotes the total number of tamperings, then
$$\Ex[\TDist] \leq \frac{1 - \mu + n\cdot \gamma}{\tau-2\gamma}.$$
\end{claim}

\begin{proof}
For any $\pfix{v}{i-1} \in \Case_1$, we have
\begin{align*}
    \Ex_{v_i \gets \vDist_i \mid \pfix{v}{i-1}} [\gain{v}{i}] \geq \gainAppMax{v}{i-1} -\gamma \geq \tau -\gamma = (\tau - \gamma)\cdot \Pr_{u_i\gets \uDist}[\Case_1(\pfix{v}{i-1})\lor \Case_2(\pfix{v}{i-1},u_i)],\stepcounter{equation}\tag{\theequation}\label{eq:eq101}
\end{align*}
and, for any $\pfix{v}{i-1} \not \in \Case_1$ we have 
\remove{
\Snote{Following is the proof with full details. Feel free to delete or combine steps.}

\begin{align*} 
    &\Ex_{u_i \gets \vDist_i \mid \pfix{u}{i-1}} \left[\gain{u}{i}\right]\\
    &\geq \Ex_{u_i \gets \vDist_i \mid \pfix{u}{i-1}} \left[\gainApp{u}{i} -\gamma\right]\\
    &= \Ex_{u_i \gets \vDist_i \mid \pfix{u}{i-1}} \Big[\gainApp{u}{i} \mid \overline{\Case_3} \wedge \overline{\Case_2}\Big]\cdot \Pr_{u_i \gets \vDist_i \mid \pfix{u}{i-1}}[\overline{\Case_3} \wedge \overline{\Case_2}]+
    \Ex_{u_i \gets \vDist_i \mid \pfix{u}{i-1}} \Big[\gainApp{u}{i} \mid \Case_3\Big]\cdot \Pr_{u_i \gets \vDist_i \mid \pfix{u}{i-1}}[\Case_3] -\gamma\\
    &=\Ex_{u_i \gets \uDist} \Big[\gainApp{u}{i} \mid \overline{\Case_3} \wedge \overline{\Case_2}\Big]\cdot \Pr_{u_i \gets \uDist}[\overline{\Case_3} \wedge \overline{\Case_2}]+
    \gainAppMax{u}{i-1}\cdot (\Pr_{u_i \gets \uDist}[\Case_3] + \Pr_{u_i \gets \uDist}[\Case_2]) -\gamma\\
    &= \Ex_{u_i \gets \uDist} \Big[\gainApp{u}{i}\Big]+  \Ex_{u_i \gets \uDist} \Big[\gainAppMax{u}{i-1} - \gainApp{u}{i}\mid \Case_2 \Big]\cdot \Pr_{u_i \gets \uDist}[\Case_2] -\gamma\\
    &=  \Ex_{u_i \gets \uDist} \Big[\gainAppMax{u}{i-1} - \gainApp{u}{i} \mid \Case_2 \Big]\cdot \Pr_{u_i \gets \uDist}[\Case_2] -\gamma\\
    &\geq (\tau-\gamma) \cdot  \Pr_{u_i \gets \uDist}[\Case_2] -\gamma\\
    &=(\tau-\gamma)\cdot \Ex_{u_i \gets \uDist}[\Case_2(\pfix{u}{i})]-\gamma.\stepcounter{equation}\tag{\theequation}\label{eq:eq102}
\end{align*}
}
\begin{align*} 
    \Ex_{v_i \gets \vDist_i \mid \pfix{v}{i-1}} \left[\gain{\pfix{v}{i}}{}\right]
    &\geq\Ex_{v_i \gets \vDist_i \mid \pfix{v}{i-1}} \left[\gainApp{v}{i}-\gamma\right]\\
    &= \Ex_{u_i \gets \uDist}[(1-\Case_2(\pfix{v}{i-1},u_i))\cdot \gainApp{\pfix{v}{i-1},u_i}{} + \Case_2(\pfix{v}{i-1},u_i) \cdot \gainAppMax{v}{i-1}] -\gamma \\
    &\geq \Ex_{u_i \gets \uDist}[(1-\Case_2(\pfix{v}{i-1},u_i))\cdot \gainApp{\pfix{v}{i-1},u_i}{} + \Case_2(\pfix{v}{i-1},u_i) \cdot (\gainApp{\pfix{v}{i-1},u_i}{} + \tau -2\gamma)] \\
    &= \Ex_{u_i \gets \uDist}[ \gainApp{\pfix{v}{i-1},u_i}{} + \Case_2(\pfix{v}{i-1},u_i) \cdot (\tau-2\gamma)] \\
    &\geq -\gamma + \Ex_{u_i \gets \uDist}[\Case_2(\pfix{v}{i-1},u_i) \cdot  (\tau-2\gamma)] \\
    &= (\tau-2\gamma) \cdot  \Pr_{u_i \gets \uDist}[\Case_2(\pfix{v}{i-1},u_i)]-\gamma\\
    &= (\tau-2\gamma) \cdot  \Pr_{u_i \gets \uDist}[\Case_1(\pfix{v}{i-1}) \lor \Case_2(\pfix{v}{i-1},u_i)] -\gamma.\stepcounter{equation}\tag{\theequation}\label{eq:eq102}
\end{align*}
By Equations~\ref{eq:eq101} and~\ref{eq:eq102}, for any $\pfix{v}{i-1} \in \Supp(\uDist)^{i-1}$ we have
\begin{equation} \label{eq:eq103}
    \Ex_{v_i \gets \vDist_i \mid \pfix{v}{i-1}} \left[\gain{v}{i}\right] \geq (\tau -2\gamma) \cdot   \Pr_{u_i \gets \uDist}[\Case_1(\pfix{v}{i-1}) \lor \Case_2(\pfix{v}{i-1},u_i)] -\gamma.
\end{equation}
Also,  let $\PrTam(\pfix{v}{i-1})$ be the probability of tampering in the $i$th block conditioned on the prefix $\pfix{v}{i-1}$. Namely, 
$$\PrTam(\pfix{v}{i-1}) =  \Pr_{u_i \gets \uDist}[\Case_1(\pfix{v}{i-1}) \lor \Case_2(\pfix{v}{i-1},u_i)].$$
The definition of $\PrTam(\pfix{v}{i-1})$ together with Equation~\ref{eq:eq103} implies that
\begin{equation} \label{eq:eq104}
    \Ex_{v_i \gets \vDist_i \mid \pfix{v}{i-1}} \left[\gain{v}{i}\right] \geq (\tau -2\gamma) \cdot  \PrTam(\pfix{v}{i-1}) -\gamma.
\end{equation}
We now obtain that
\begin{align*}
    \Ex\left[f(\vDistVec)\right] - \mu &= \Ex\left[\sum_{i=1}^n \gain{\vDist}{i}\right]\\
    \text{(by linearity of expectation)}&= \sum_{i=1}^n \Ex\left[\gain{\vDist}{i}\right]\\
    &= \sum_{i=1}^n \Ex_{\pfix{v}{i-1}\gets\pfix{\vDist}{i-1}}\left[\Ex_{v_i \gets \vDist_i \mid \pfix{v}{i-1}} \left[\gain{v}{i}\right]\right]\\
    \text{(by Equation~\ref{eq:eq104})~~}&\geq   (\tau-2\gamma) \cdot \sum_{i=1}^n \Ex_{\pfix{v}{i-1}\gets\pfix{\vDist}{i-1}}\left[\PrTam(\pfix{v}{i-1})\right]-n\gamma\\
    \text{(by linearity of expectation)}&= (\tau-2\gamma) \cdot \Ex\left[ \sum_{i=1}^n \PrTam(\pfix{\vDist}{i-1})\right] -n\gamma\\
            &\geq(\tau -2\gamma) \cdot \Ex[\TDist] -n\gamma.
\end{align*}
Therefore, we have 
$$\Ex[\TDist]\leq \frac{1- \mu +n\cdot \gamma}{\tau-2\cdot \gamma}.$$
\end{proof}

\subsection{Polynomial-time Biasing Attack Using Probably Approximate Oracles}

In this subsection, we finally prove Theorem \ref{thm:ProdOnline} by getting rid of the promised approximate oracles, and having them approximated by the attacker itself, though only with high probability. We will also show how to pick the parameters of the attack to achieve the bounds claimed in  Theorem \ref{thm:ProdOnline}.

We start by sketching the intuitive fact that the approximate oracles of Definition~\ref{defs:biasingApproximate} could indeed be provided to our polynomial time attacker of Construction~\ref{const:Semi-Poly} with high probability by  repeated sampling and applying Chernoff bound.

\begin{construction} [Oracle $\gainApp{\cdot}{}$] \label{cons:appgain} Given a prefix $\pfix{v}{i} \in \Supp(\uDist)^i$, let 
$$k=-12\cdot \frac{\ln(\gamma/2) +\ln(\ln(1+\gamma))-\ln(-\ln(\gamma/2))}{\gamma^2}.$$
Sample $k$ random continuations $\ol{w}_1^1,\dots,\ol{w}_1^k$ from $\uDist^{n-i-1}$ and $k$  continuations $\ol{w}_2^1,\dots,\ol{w}_2^k$ from $\uDist^{n-i}$. Let 
$$\avrApp{v}{i} = \frac{1}{k}\cdot\left(f(\pfix{v}{i}, \ol{w}_1^1) + \dots + f(\pfix{v}{i}, \ol{w}_1^k) \right)$$
$$\avrApp{v}{i-1} = \frac{1}{k}\cdot\left(f(\pfix{v}{i-1}, \ol{w}_2^1) + \dots + f(\pfix{v}{i-1}, \ol{w}_2^k) \right)$$ 
and output $\gainApp{v}{i} = \avrApp{v}{i} - \avrApp{v}{i-1}.$
\end{construction}
\begin{claim}\label{clm:gainApp}
For the oracle $\gainApp{\cdot}{}$  of Construction~\ref{cons:appgain} We have 
$$\Pr[|\gainApp{y}{i} - \gain{y}{i}| \geq \gamma] \leq\frac{-\gamma\cdot \ln(1+\gamma)}{2\cdot \ln(\gamma/2)} \leq \frac{\gamma}{2}.$$
\end{claim}

\begin{proof}
Define independent Boolean random variables $\fDist_1^1,\dots,\fDist_1^k$ and $\fDist_2^1,\dots,\fDist_2^k$  such that for each $j\in [k]$ we have $\fDist_1^j \equiv f(\pfix{v}{i}, \uDist^{n-i})$ and $\fDist_2^j \equiv f(\pfix{v}{i-1}, \uDist^{n-i+1})$. Also let $\avr{y}{i}=\Ex[f(\pfix{v}{i}, \uDist^{n-i})]$ and $\avr{y}{i-1}=\Ex[f(\pfix{v}{i-1}, \uDist^{n-i+1})]$. By  Chenroff inequality we have,
\begin{equation}\label{ineq:chf1}
    \Pr\Big[\big|\frac{1}{k}\cdot\sum_{j\in k} \fDist_1^j - \avr{y}{i} \big| \geq \frac{\gamma}{2}\Big]=\Pr\left[|\avrApp{y}{i} - \avr{y}{i}| \geq \frac{\gamma}{2}\right] \leq \e^\frac{-k\cdot\gamma^2}{12}.
\end{equation}
On the other hand, again by Chenroff we have,
\begin{equation}\label{ineq:chf2}
    \Pr\Big[\big|\frac{1}{k}\cdot\sum_{j\in k} \fDist_2^j - \avr{y}{i}\big| \geq \frac{\gamma}{2}\Big]=\Pr\Big[\big|\avrApp{y}{i-1} - \avr{y}{i-1}\big| \geq \frac{\gamma}{2}\Big] \leq \e^\frac{-k\cdot\gamma^2}{12}.
\end{equation}
Now by Inequalities~\ref{ineq:chf1} and~\ref{ineq:chf2}, we have
\begin{equation*}
    \Pr\big[|\gainApp{y}{i} - \gain{y}{i}| \geq \gamma\big] \leq 2\cdot \e^\frac{-k\cdot\gamma^2}{12} = \frac{-\gamma\cdot \ln(1+\gamma)}{2\cdot \ln(\gamma/2)}.
\end{equation*}
Note that
$\frac{- \ln(1+\gamma)}{\ln(\gamma/2)}$ is less than $1$ for $\gamma\in[0,1]$.Therefore, we have,
$$\Pr\big[|\gainApp{y}{i} - \gain{y}{i}| \geq \gamma\big] \leq \frac{\gamma}{2}.$$
\end{proof}

\begin{claim}\label{clm:gainAppTime}
The implementation of the oracle of Construction \ref{cons:appgain} runs in time $O(n\cdot \ell/\gamma^3)$.
\end{claim}
\begin{proof}
The oracle samples $$k=-12\cdot \frac{\ln(\gamma/2) +\ln(\ln(1+\gamma))-\ln(-\ln(\gamma/2))}{\gamma^2} \leq \frac{24}{\gamma^3}$$ 
continuations. Therefore, the running time is $O(n\cdot \ell/\gamma^3)$.
\end{proof}

\begin{construction}[Oracles $\gainAppMax{\cdot}{}$ and $\AppArgMax{\cdot}{}$] \label{cons:appMaxgain} Given a prefix $\pfix{v}{i-1} \in \Supp(\uDist)^{i-1}$, sample $k=\frac{-\ln(\gamma/2)}{\ln(1+\gamma)}$ blocks $u^1,\dots,u^k$ from $\uDist$. Now let 
$$\AppArgMax{v}{i-1}=\argmax_{u^j} {\gainAppMax{\pfix{v}{i-1},u^j}{}} \text{ ~~~and~~~ } \gainAppMax{v}{i-1}=\max_{u^j} {\gainApp{\pfix{v}{i-1},u^j}{}}.$$ 
\end{construction}

\begin{claim}\label{clm:gainAppMax1}
Let $\lambda \in[-1,1]$ be such that $\Pr\left[\gain{\pfix{y}{i-1},\uDist}{}\geq \lambda\right] \geq \gamma$. For the $\gainAppMax{\cdot}{}$ oracle of Construction~\ref{cons:appMaxgain} We have 
$$\Pr\left[\gainAppMax{y}{i-1} \leq \lambda -\gamma \right] \leq \gamma.$$ 
\end{claim}
\begin{proof}
We first bound the probability of all the actual gains being less then $\lambda$. We have
$$\Pr[\forall j\in [k], \gain{\pfix{v}{i-1},u^j}{}\leq \lambda] = \Pr[\gain{\pfix{v}{i-1},\uDist}{} \leq -\gamma]^k \leq (1-\gamma)^k \leq \frac{\gamma}{2}.$$
Now, consider the following event 
$$B = \begin{cases}
0 & \text{if for all $j\in k$ we have $|\gainApp{\pfix{v}{i-1},u^j}{} - \gain{\pfix{v}{i-1},u^j}{}|\leq \gamma$,}\\
1 & \text{otherwise.}
\end{cases}$$
By Claim~\ref{clm:gainApp} and a union bound we have $\Pr[B]\leq k\cdot \frac{-\gamma\cdot \ln(1+\gamma)}{2\cdot \ln(\gamma)} = \frac{\gamma}{2}$. Therefore, we have
$$\Pr[\forall j\in [k], \gainApp{\pfix{v}{i-1},u^j}{}\leq \lambda -\gamma] \leq \Pr[\forall j\in [k], \gain{\pfix{v}{i-1},v^j}{}\leq \lambda] + \frac{\gamma}{2} \leq  \gamma.$$
\end{proof}

\begin{claim}\label{clm:gainAppMax2} For the oracle $\gainAppMax{\cdot}{}$  of Construction~\ref{cons:appMaxgain}, we have $\Pr[\gainAppMax{y}{i-1} \leq -2\gamma] \leq \gamma$.
\end{claim}
\begin{proof}
We first bound the probability of all the actual gains being less then $-\gamma$. We have
$$\Pr[\forall j\in [k], \gain{\pfix{v}{i-1},u^j}{}\leq -\gamma] = \Pr[\gain{\pfix{v}{i-1},\uDist}{} \leq -\gamma]^k \leq \left(\frac{1}{1+\gamma}\right)^k = \frac{\gamma}{2}.$$
Now consider the following event 
$$B = \begin{cases}
0 & \text{if for all $j\in k$ we have $|\gainApp{\pfix{v}{i-1},u^j}{} - \gain{\pfix{v}{i-1},u^j}{}|\leq \gamma$,}\\
1 & \text{otherwise.}
\end{cases}$$
By Claim~\ref{clm:gainApp} and a union bound, we have $\Pr[B]\leq k\cdot \frac{-\gamma\cdot\ln(1+\gamma)}{\ln(\gamma/2)} = \frac{\gamma}{2}$. Therefore, we have
$$\Pr[\forall j\in [k], \gainApp{\pfix{v}{i-1},u^j}{}\leq -2\gamma] \leq \Pr[\forall j\in [k], \gain{\pfix{v}{i-1},u^j}{}\leq -\gamma] + \frac{\gamma}{2} \leq  \gamma.$$
\end{proof}

\begin{claim}\label{clm:gainAppMaxTime}
The Oracles of Construction \ref{cons:appMaxgain} run in time $O(n\cdot \ell/\gamma^5)$.
\end{claim}
\begin{proof}
The oracles generate $$k=-\frac{\ln(\gamma/2)}{\ln(1+\gamma)} \leq \frac{1}{\gamma^2}$$ samples and for each sample call the oracle $\gainApp{\cdot}{}$ of Construction \ref{cons:appgain}. Therefore, by Claim \ref{clm:gainAppTime} the running time of the Oracles of Construction \ref{cons:appMaxgain} are $O(n\cdot \ell/\gamma^5)$
\end{proof}

Now we show that using the approximate oracles of Constructions \ref{cons:appgain} and \ref{cons:appMaxgain} we still can achieve the desired bounds.
\begin{construction}[Attack using probably approximate oracles] \label{const:Poly} 
Given input vector $(v_1,\dots,v_{i-1},u_i)\in \Supp(\uDist)^i$, $\AppTam$ will output some $v_i \in \Supp(\uDist)$. Let $w_i =\AppArgMax{v}{i-1}$ as defined in Definition~\ref{defs:biasingApproximate}. Now, $\EffTam$ chooses its output $v_i$ as follows.
\begin{itemize}
    \item {\bf Tampering.} If $\gainAppMax{\pfix{v}{i-1}}{} \geq \tau$ or if $\gainApp{\pfix{v}{i-1},u_i}{} \leq -\tau$, then output $v_i = w_i$. 
    \item  {\bf Not tampering.} Otherwise, output the original sample $v_i = u_i$.
\end{itemize}
where $\gainAppMax{\cdot}{}$,$\AppArgMax{\cdot}{}$ and $\gainApp{\cdot}{}$ oracles are instantiated using Constructions \ref{cons:appMaxgain} and \ref{cons:appgain}.
\end{construction}

The following claim bounds the average of function in presence of the attack of Construction \ref{const:Poly}.

\begin{claim} \label{clm:EffBias} If $\vDistVec \equiv \twist{\uDist^n}{\EffTam}$ is the tampering distribution of the efficient attacker $\EffTam$ of Construction~\ref{const:Semi-Poly}, then it holds that
$$\Ex[f(\vDistVec)] \geq 1 -  \e^{\frac{-(\mu-2n\cdot \gamma)^2}{2n\cdot(\tau+4\gamma)^2}} - 4n\cdot \gamma.$$
\end{claim}
\begin{proof}
Let $B$ be the event that for at least one of the queries of Construction \ref{const:Poly}, the promises are not satisfied. Namely,
\begin{align*}
B&= \Exists i, \Pr_{u_i \gets \uDist}\big[\gain{\pfix{v}{i-1},{u_i}}{} > \gainAppMax{v}{i-1} + 2\gamma\big] \leq \gamma\\
&~~\lor  \gainAppMax{u}{i-1} < -2\gamma\\
&~~\lor |\gainApp{\pfix{v}{i-1},u_i}{} - \gain{\pfix{v}{i-1},u_i}{}| \geq \gamma.
\end{align*}
During the whole course of process of Construction \ref{const:Poly}, there are exactly $n$, $\gainAppMax{\cdot}{}$ queries and exactly $\gainApp{\cdot}{}$ queries. Therefore, using Claims \ref{clm:gainApp}, \ref{clm:gainAppMax1} and \ref{clm:gainAppMax2} we have
$$\Pr[B] \leq n\cdot \gamma + n\cdot \gamma + n\cdot \frac{\gamma}{2}\leq 3\cdot n\cdot\gamma.$$
We know that conditioned on $\overline{B}$, the attack of Construction \ref{const:Poly} will be identical to the attack of Construction \ref{const:Semi-Poly} and increases the average to $1 -  \e^{\frac{-(\mu-2n\cdot \gamma)^2}{2n\cdot(\tau+4\gamma)^2}} - n\cdot \gamma$ by Claim \ref{clm:AppBias}. Therefore, the attack of Construction \ref{const:Poly} can unconditionally increases the average to at least 
$$1 -  \e^{\frac{-(\mu-2n\cdot \gamma)^2}{2n\cdot(\tau+4\gamma)^2}} - n\cdot \gamma - \Pr[B] \geq 1 -  \e^{\frac{-(\mu-2n\cdot \gamma)^2}{2n\cdot(\tau+4\gamma)^2}} - 4n\cdot \gamma.$$
\end{proof}

Now we state and prove another claim that bounds the expected number of tamperings performed by the attack of Construction \ref{const:Poly}

\begin{claim} \label{clm:EffTam}
For a tampering sequence $\vVec \gets \twist{\uVec}{\EffTam}$, let  $T_i = \Case_1(\pfix{v}{i-1}) \lor \Case_2(\pfix{v}{i-1},u_i)$ be the event (or equivalently the  Boolean function) denoting that a tampering choice is made by the adversary of Construction \ref{const:Poly} over the $i$'th block. If $T=\sum_{i=1}^n T_i$ denotes the total number of tamperings, then
$$\Ex[\TDist] \leq \frac{1 - \mu + n\cdot \gamma}{\tau-2\gamma}+ 3n^2\cdot \gamma.$$
\end{claim}
\begin{proof}
Define event $B$ similar to the proof of Claim \ref{clm:EffBias}. 
We know that conditioned on $\overline{B}$, the attack of Construction \ref{const:Poly} will be identical to the attack of Construction \ref{const:Semi-Poly} and uses average budget at most $\frac{1 - \mu + n\cdot \gamma}{\tau-2\gamma}$ by Claim \ref{clm:AppTam2}. Therefore, the attack of Construction \ref{const:Poly} will use average budget atmost
$$\frac{1 - \mu + n\cdot \gamma}{\tau-2\gamma} +n\cdot\Pr[B] \leq \frac{1 - \mu + n\cdot \gamma}{\tau-2\gamma} + 3n^2\cdot \gamma$$
\end{proof}

Finally, we prove  Theorem \ref{thm:ProdOnline} by relying on the probabilistic guarantees of the  estimators of the above constructions for the approximate oracles.
\begin{proof}[Proof of Theorem \ref{thm:ProdOnline}]
Let 
$$k=\ln\left(\frac{2}{1-\rho}\right) \text{ and } \tau = \frac{\mu}{1.9\sqrt{kn}} \text{ and } \gamma = \min \set{ \frac{\mu}{20n}, \frac{\mu}{80\sqrt{k\cdot n}}, \frac{1-\rho}{8\cdot n}, \frac{\sqrt{\ln(2/(1-\rho))}}{3n\sqrt{n}}}.$$ 
Also, let $\EffTam_{\tau,\gamma}$ be the attack of Construction \ref{const:Poly} instantiated with the parameters $\tau$ and $\gamma$ specified above. If $\vDistVec \equiv \twist{\uDist^n}{\EffTam}$, by Claim \ref{clm:EffBias} we have
\begin{align*}
\Ex[f(\vDistVec)] &\geq 1 -  \e^{\frac{-(\mu-2n\cdot \gamma)^2}{2n\cdot(\tau+4\gamma)^2}} - 4n\cdot \gamma\\
&\geq 1 -  \e^{\frac{-(0.9\mu)^2}{2n\cdot(1.1\tau)^2}} - \frac{1-\rho}{2}\\
&\geq 1 - \frac{1-\rho}{2} -\frac{1-\rho}{2}\\
&=\rho .
\end{align*}
On the other hand, By Claim \ref{clm:EffTam} we have
\begin{align*}
    \Ex[\TDist] &\leq \frac{1 - \mu + n\cdot \gamma}{\tau-2\gamma}+ 3n^2\cdot \gamma\\
    &\leq  \frac{1 - 0.95\mu}{0.95\tau} +\sqrt{n\cdot\ln(2/(1-\rho))}\\
    &\leq \frac{2 - 1.9\mu}{\mu}\cdot \sqrt{n\cdot\ln(2/(1-\rho))} + \sqrt{n\cdot\ln(2/(1-\rho))}\\
    &\leq \frac{2}{\mu}\cdot \sqrt{n\cdot\ln(2/(1-\rho))}.
\end{align*}
We also know that $\gamma = \omega({\mu\cdot(1-\rho)}/{n^2})$. Therefore, by Claims \ref{clm:gainAppTime} and \ref{clm:gainAppMaxTime} we conclude that the running time of $\EffTam_{\tau,\gamma}$ is $O({n^{12}\cdot \ell}/{(\mu\cdot(1-\rho))^5})$.
\end{proof}

\section{Conclusion and Some Questions}

In this work, we proved strong barriers against basing the robustness of learning algorithms in both settings of evasion attacks (who aim at fining adversarial examples) and poisoning attacks (who try to increase the misclassification probability by minimally tampering with the training data) on computational hardness assumptions. Namely, we showed that  a broad set of learning problems, and in particular  classification tasks whose instances are drawn from a product distribution of dimension $n$, are inherently vulnerable to \emph{polynomial-time} adversaries that find adversarial examples of (sublinear) Hamming distance $O(\sqrt n)$ and increase the classification error from $1\%$ to $99\%$.

Our proofs are based on a new coin tossing attack that is inspired by the recent results of \cite{RazCoin2018}. Our coin tossing attacks could be interpreted as a polynomial-time algorithmic proof of the well-known classical result of concentration measure in product distributions \cite{amir1980unconditional,milman1986asymptotic,mcdiarmid1989method,talagrand1995concentration}. While, concentration of measure guarantees that for any initial large enough set $\cS$, ``most'' points in the probability space are ``close'' to $\cS$, our algorithmic proof shows how to find such close instances in polynomial time. 

Our work motivates studying the following natural questions.

\paragraph{Other \emph{computationally} concentrated metric probability spaces (e.g., computational  \Levy families)?} Normal \Levy families include numerous metric probability spaces in which the measure is concentrated. Our Theorem \ref{thm:ProdOnline} shows that product distributions, as a special form of normal \Levy families are computationally concentrated. How about other spaces (e.g., isotropic Gaussian distribution under Euclidean distance) for which information theoretic concentration of measure is known? Proving any such results about any other metric probability space, directly leads to polynomial-time attacks against evasion robustness of any learning problem whose instances are drawn from those spaces. More generally, for what metric probability spaces can we prove nontrivial computational concentration of measure. Namely, a space would be computationally concentrated, if for any set $\cS$ with ``large enough'' measure, ``most'' of the points in the probability space can be \emph{efficiently} mapped to ``close'' points in $\cS$. Here, ``close'' shall be something that beats the trivial bound given by the diameter of the space. The latter question can be also interpreted as proving (approximate) ``computational isoperimetric'' inequalities. Namely, here the ``computational boundary'' of the set $\cS$ would be defined by a computationally bounded mapping that efficiently maps the points in ``computational distance'' $b$ to $\cS$. As in Theorem \ref{thm:ProdOnline}, access to the underlying probability distribution and membership in $\cS$ can be provided by   sampling and membership oracles.

\paragraph{Handling smaller initial error.} Our Theorem \ref{thm:ProdOnline} gives nontrivial attacks with sublinear $o(n)$ perturbation only if the initial average (or equivalently the probability of the target set $\cS$) is at least $ \omega(\log n / \sqrt n)$. However, the information theoretic attacks of \cite{mahloujifar2018curse} can handle original error that is only substantially $\exp(-o(n))$ large. What is the smallest probability $\mu$ for which we can obtain computational concentration of measure (for product or other distributions) if the target set has probability at least $\mu$? Solving this problem would have immediate impact on polynomial time evasion (and poisoning--in case of product distributions) attacks that attack systems with smaller error (in the no-attack case).

\newpage

\bibliographystyle{alpha}
\newcommand{\etalchar}[1]{$^{#1}$}

\end{document}